%% file: main.tex
\documentclass{ecai} 

\usepackage[utf8]{inputenc}
\usepackage[small]{caption}
\usepackage{subcaption}
\usepackage{amsmath}
\usepackage{amsthm}
\usepackage{amsfonts}
\usepackage{booktabs}
\usepackage{algorithm}
\usepackage{algpseudocode}
\usepackage[switch]{lineno}
\usepackage{xspace}
\usepackage[dvipsnames,table]{xcolor}
\usepackage{graphicx}
\usepackage{lipsum}
\usepackage{scalerel}
\usepackage{amssymb}
\usepackage{makecell}
\usepackage{bm}
\usepackage{tikz}
\usetikzlibrary{calc}
\usepackage{array}
\newcolumntype{P}[1]{>{\centering\arraybackslash}p{#1}}
\newcolumntype{M}[1]{>{\centering\arraybackslash}m{#1}}

\DeclareMathOperator*{\Forall}{\scalerel*{\forall}{\sum}}

\newcommand{\BibTeX}{B\kern-.05em{\sc i\kern-.025em b}\kern-.08em\TeX}

\begin{document}
\include{macros}


\begin{frontmatter}


\paperid{2243} 


\title{On Learning Action Costs from Input Plans}


\author{\fnms{Marianela}~\snm{Morales}\thanks{Corresponding Author. Email: marianela.moraleselena@jpmorgan.com}\footnote{Equal contribution.}}
\author{\fnms{Alberto}~\snm{Pozanco}\footnotemark}
\author{\fnms{Giuseppe}~\snm{Canonaco}} 
\author{\fnms{Sriram}~\snm{Gopalakrishnan}}
\author{\fnms{Daniel}~\snm{Borrajo}}
\author{\fnms{Manuela}~\snm{Veloso}}

\address{J.P. Morgan AI Research}


\begin{abstract}
Most of the work on learning action models focus on learning the actions' dynamics from input plans.
This allows us to specify the valid plans of a planning task.
However, very little work focuses on learning action costs, which in turn allows us to rank the different plans.
In this paper we introduce a new problem: that of learning the costs of a set of actions such that a set of input plans are optimal under the resulting planning model.
To solve this problem we present \optimalapproach{k}, an algorithm to learn action's costs from unlabeled input plans.
We provide theoretical and empirical results showing how \optimalapproach{k} can successfully solve this task.
\end{abstract}

\end{frontmatter}


\input{sections/01-introduction}
\input{sections/02-background}
\input{sections/04-third_option}
\input{sections/04b-algorithm}
\input{sections/05-evaluation}
\input{sections/06-related-work}
\input{sections/07-conclusions}

\newpage 

\section*{Disclaimer}
This paper was prepared for informational purposes by the Artificial Intelligence Research group of JPMorgan Chase \& Co. and its affiliates ("JP Morgan'') and is not a product of the Research Department of JP Morgan. JP Morgan makes no representation and warranty whatsoever and disclaims all liability, for the completeness, accuracy or reliability of the information contained herein. This document is not intended as investment research or investment advice, or a recommendation, offer or solicitation for the purchase or sale of any security, financial instrument, financial product or service, or to be used in any way for evaluating the merits of participating in any transaction, and shall not constitute a solicitation under any jurisdiction or to any person, if such solicitation under such jurisdiction or to such person would be unlawful.

\bibliography{ecai25}

\newpage

\include{appendix}

\end{document}

%% file: macros.tex
\newcommand{\marianela}[1]{\textcolor{orange}{#1}}
\newcommand{\alberto}[1]{\textcolor{blue}{#1}}
\newcommand{\giuseppe}[1]{\textcolor{violet}{#1}}
\newcommand{\ram}[1]{\textcolor{cyan}{#1}}
\newcommand{\daniel}[1]{\textcolor{red}{#1}}

\newcommand{\fluents}{\ensuremath{F}\xspace}
\newcommand{\actions}{\ensuremath{A}\xspace}
\newcommand{\relevantactions}{\ensuremath{A^\multipleplans}\xspace}
\newcommand{\init}{\ensuremath{I}\xspace}
\newcommand{\goal}{\ensuremath{G}\xspace}
\newcommand{\costfunction}{\ensuremath{C}\xspace}
\newcommand{\initcostfunction}{\ensuremath{\bar{C}}\xspace}
\newcommand{\cost}{\ensuremath{c}\xspace}
\newcommand{\initcost}{\ensuremath{\bar{c}}\xspace}
\newcommand{\task}{\ensuremath{{\cal P}}\xspace}
\newcommand{\plan}{\ensuremath{\pi}\xspace}
\newcommand{\plans}{\ensuremath{\Pi}\xspace}
\newcommand{\stripstask}{\ensuremath{\task=\langle \fluents, \actions, \init, \goal, \costfunction \rangle}\xspace}
\newcommand{\state}{\ensuremath{s}\xspace}
\newcommand{\allstates}{\ensuremath{\cal S}\xspace}
\newcommand{\reachablestates}{\ensuremath{{\cal S}_R}\xspace}
\newcommand{\action}{\ensuremath{a}\xspace}
\newcommand{\precondition}{\ensuremath{\textsc{pre}}\xspace}
\newcommand{\effects}{\ensuremath{\textsc{eff}}\xspace}
\newcommand{\addeffects}{\ensuremath{\textsc{add}}\xspace}
\newcommand{\deleffects}{\ensuremath{\textsc{del}}\xspace}
\newcommand{\actionapplication}{\ensuremath{\gamma}\xspace}
\newcommand{\planapplication}{\ensuremath{\Gamma}\xspace}
\newcommand{\multipletasks}{\ensuremath{{\cal T}}\xspace}
\newcommand{\multipleplans}{\ensuremath{{\cal M}}\xspace}

\newcommand{\taskc}{\ensuremath{{\cal P}^{\bar{C}}}\xspace}
\newcommand{\multipletasksc}{\ensuremath{{\cal T}^{\bar{C}}}\xspace}
\newcommand{\stripstaskc}{\ensuremath{\taskc=\langle \fluents, \actions, \init, \goal, \costfunction \rangle}\xspace}

\newcommand{\cfl}{\ensuremath{\textsc{cfl}}\xspace}
\newcommand{\mcf}{\ensuremath{\textsc{mcf}}\xspace}
\newcommand{\scf}{\ensuremath{\textsc{scf}}\xspace}
\newcommand{\icf}{\ensuremath{\textsc{icf}}\xspace}
\newcommand{\icfs}{\ensuremath{\textsc{icf}\text{s}}\xspace}
\newcommand{\cflc}{\ensuremath{\textsc{cfl}^{\bar{C}}}\xspace}
\newcommand{\mcfc}{\ensuremath{\textsc{mcf}^{\bar{C}}}\xspace}
\newcommand{\scfc}{\ensuremath{\textsc{scf}^{\bar{C}}}\xspace}

\newcommand{\approach}{\textsc{lacfip}\xspace}
\newcommand{\optimalapproach}[1]{\ensuremath{\approach^{#1}}\xspace}
\newcommand{\greedyapproach}{\ensuremath{\approach^g}\xspace}
\newcommand{\baseline}{\textsc{baseline}\xspace}

\newtheorem{definition}{Definition}
\newtheorem{theorem}{Theorem}
\newtheorem{proposition}{Proposition}
\newtheorem{lemma}{Lemma}
\newtheorem{corollary}{Corollary}
\newtheorem{example}{Example}
\newtheorem{remark}{Remark}
\theoremstyle{definition}
\newcommand{\alternativeplans}[1]{\ensuremath{\plans^{#1}}\xspace}
\newcommand{\alternativeplan}{\ensuremath{\plan^{\mathsf{a}}}\xspace}
\newcommand{\addaction}[3]{\mathsf{({#1}\mbox{ }{#2}\mbox{ }{#3})}}
\newcommand{\abs}[1]{\lvert#1\rvert}

%% file: sections/01-introduction.tex
\section{Introduction}
Classical planning is the task of choosing and organizing a sequence of deterministic actions such that, when applied in a given initial state, it results in a goal state~\cite{DBLP:books/daglib/0014222}.
Most planning works assume the actions' dynamics or \emph{domain model}, i.e., how actions change the state, are provided as an input and turn the focus to the efficient synthesis of plans.
This is a strong assumption in many real-world planning applications, where domain modeling is often a challenging task~\cite{kambhampati2007model}.
Motivated by the difficulty of crafting action models, several works have tried to automatically learn domain models from input observations~\cite{yang2007learning,gregory2016domain,arora2018review,aineto2019learning,gragera2023planning}. 
Although they make different assumptions on the type of observations (full or partial plan, access to intermediate states, noisy observations, etc.), most works solely focus on learning the actions' dynamics but not their associated cost. 

In this paper we argue that learning action costs is as important as learning the actions' dynamics.
While actions' dynamics allow us to determine the \emph{validity} of traces in a domain model, the actions costs allow us to get the \emph{quality} of each of these traces, which is needed whenever we want to generate good plans.
Moreover, there are many real-world planning applications where the actions' dynamics are known, but their cost is either unknown and we aim to learn it from scratch; or approximate and we aim to refine it.
In both cases we can use data in the form of observed plans to acquire this knowledge.

Consider the case of a navigation tool that suggests routes to drivers.
In this domain the actions' dynamics are clear: cars can move through different roads and taking an action (i.e., taking an exit) will change the car's position.
The navigation tool will typically aim to generate the route with the shortest driving time, i.e., the least costly or optimal plan.
To do that, it makes some assumptions about the cost of each action (driving times): for example, being a function of the distance.
While this can be a good proxy, it can be further refined by observing the actual plans coming from users of the navigation tool.
By observing these plans we cannot only get more accurate driving times, but also understand which routes users prefer and adjust the model accordingly. 
Financial planning is yet another example where we have access to many plans, actions' dynamics are known, but properly estimating their cost for different people is crucial and challenging.
\citeauthor{pozanco2023combining} (\citeyear{pozanco2023combining}) aim to generate realistic financial plans by maximizing their likelihood.
To do that, they assign lower costs to more likely actions, i.e., saving $\$5$ in memberships is less costly than increasing the salary by $\$1,000$.
\citeauthor{pozanco2023combining} mention that these costs can be given or inferred from data but do not provide further details on how to do it.
Like in the navigation tool case, here we could gather observed plans on how users are saving and spending money to achieve their financial goals.
By doing this, we could more accurately assign costs to each action so as to generate plans that better align with user preferences.

In this paper we introduce a new problem: that of learning the costs of a set of actions such that a set of input plans are optimal under the resulting planning model.
We formally prove that this problem does not have a solution for an arbitrary set of input plans, and relax the problem to accept solutions where the number of input plans that are turned optimal is maximized.
We also present variations of this problem, where we guarantee input plans are the only optimal plans allowed; or we try to minimally modify an existing cost function instead of learning it from scratch.
We then introduce \optimalapproach{k}, a common algorithm to solve these tasks.
Empirical results across different planning domains show how \optimalapproach{k} can be used in practice to learn (or adapt) cost functions from unlabeled input plans.
We close the paper by drawing the connection between learning planning action costs and the related field of Inverse Reinforcement Learning~\cite{ng2000algorithms}, and by outlining some directions for future work.

%% file: sections/02-background.tex
\section{Preliminaries}

A classical planning task can be defined as follows:
\begin{definition}\label{def:strips-plan-task}
  A {\sc strips} \textbf{planning task}
can be defined as a tuple \stripstask, where \fluents is a set of fluents, \actions is a set of
 actions, $\init \subseteq \fluents$ is an initial state, $\goal\subseteq \fluents$ is a goal state, and $\costfunction: \actions \mapsto \mathbb{N}^+_1$ is the cost function that associates a cost to each action\footnote{We assume natural action costs w.l.o.g.}.  
\end{definition}

A state $\state \subseteq \fluents$ is a set of fluents that are true at a given time.
With \allstates we refer to all the possible states defined over \fluents. 
Each action $\action \in \actions$ is described by a set of preconditions $\precondition(\action)$, add effects $\addeffects(\action)$, delete effects $\deleffects(\action)$, and cost $\cost(\action)$.
An action \action is applicable in a state \state iff $\precondition(\action)\subseteq~s$.
We define the result of applying an action in a state as $\actionapplication(\state,\action)=(\state \setminus \deleffects(\action)) \cup \addeffects(\action)$.

A sequence of actions $\plan=(\action_1,\ldots,\action_n)$ is applicable in a state $\state_0$ if there are states $(\state_1.\ldots,\state_n)$ such that $\action_i$ is applicable in $\state_{i-1}$ and $\state_i=\actionapplication(\state_{i-1},\action_i)$.
The resulting state after applying a sequence of actions is $\planapplication(\state,\pi)=\state_n$, and $\cost(\plan) = \sum_{i=1}^{n} \cost(\action_i)$ denotes the cost of $\plan$.
A state $\state \in \allstates$ is reachable iff there exists a sequence of actions \plan applicable from \init such that $s \subseteq \planapplication(\init,\pi)$.
With $\reachablestates \subseteq \allstates$ we refer to the set of all reachable states of the planning task.
A sequence of actions is simple if it does not traverse the same state $\state \in \reachablestates$ more than once.

The solution to a planning task $\task$ is a plan, i.e., a sequence of actions $\plan$ such that $\goal \subseteq \planapplication(\init,\pi)$.
A plan $\plan_i$ is a subset of a plan $\plan_j$, denoted as $\plan_i\subset\plan_j$, iff the sequence of actions $\plan_i=(\action_1,\ldots,\action_n)$ is contained in the sequence of actions that conform $\plan_j$. 
We denote as $\Pi(\task)$ the set of all simple solution plans to planning task $\task$.
Also, given a plan \plan, we denote its alternatives, i.e., all the other sequence of actions that can solve \task as $\alternativeplans{\plan} = \Pi(\task) \setminus \pi$.
\begin{definition}
    A plan \plan is an \textbf{optimal plan} that solves \task iff its cost is lower or equal than that of the rest of alternative plans solving \task:
    \begin{equation}
       \small
        \cost(\plan) \leq \cost(\plan^\prime), \quad \Forall _{\plan^\prime \in \alternativeplans{\plan}}
    \end{equation}
    \label{def:optimal_plan}
\end{definition}
We will use the boolean function $\textsc{is\_optimal}(\plan,\task)$ to evaluate whether \plan optimally solves \task ($1$) or not ($0$).

%% file: sections/04-third_option.tex
\section{Learning Action Costs from Input Plans}
We are interested in learning the costs of a set of actions such that the input plans are optimal under the resulting planning model.
The underlying motivation is that by aligning the action's costs to the input plans, the new model will be able to generate new plans that better reflect the observed behavior, i.e., the user preferences.

Initially, we assume actions do not have an associated cost a priori.
To accommodate this, we extend the potential values that a cost function can have, to include empty values $\costfunction: \actions \mapsto \mathbb{N}^+_1 \cup \{\varnothing\}$.
We denote $\costfunction = \varnothing$ iff $\forall\action\in\actions,  \cost(\action) = \varnothing$. 
We then formally define a cost function learning task as follows:

\begin{definition}\label{def:cfl}
    A \textbf{cost function learning task} is a tuple $\cfl = \langle \multipletasks, \multipleplans \rangle$ where:
    \begin{itemize}
        \item $\multipletasks=(\task_1, \ldots, \task_n)$ is a sequence of planning tasks that share \fluents, \actions, and $\costfunction=\varnothing$.
        \item $\multipleplans=(\plan_1,\ldots,\plan_n)$ is a corresponding sequence of simple plans that solves \multipletasks.
    \end{itemize}
    The solution to a \cfl task is a common cost function \costfunction (common across all tasks and plans).
\end{definition}

Let us examine the problem definition.
We assume we have access to the full plan $\plan_i$ as well as the planning task $\task_i$ that it solves.
We do not impose any restriction on the planning tasks and plans: there could be repeated elements in \multipletasks or \multipleplans, and they could come from any distribution.
Unlike previous works on action's cost learning~\cite{gregory2016domain,garrido2023learning}, we do not require to know the total cost of each plan.
Moreover, we also assume that all the planning tasks and plans share the same vocabulary, i.e., they have a common set of fluents \fluents and actions \actions.
We also assume the common cost function \costfunction is initially unknown. 

The above assumptions are not restrictive and hold in many real-world applications such as the ones described in the Introduction.
For example, in the navigation scenario \fluents and \actions will remain constant as long as the city network (map) does not change, which will only occur when a new road is built.
In this domain we will get input plans with different starting points (\init) and destinations (\goal), which is supported by our problem definition.
These plans may come from users who share similar preferences, such as favoring the shortest routes, or from users with mixed preferences, where some prefer the shortest routes while others opt for scenic routes.
One could argue that in the navigation scenario it is trivial to annotate each plan with its actual duration.
While this might be true, it is clearly not so for other applications where action's cost capture probabilities or user preferences, such as  financial planning. 

Going back to Definition~\ref{def:cfl}, we purposely left open the characterization of a \cfl solution, only restricting it to be a common cost function \costfunction shared by all the input tasks.
We did this because we are interested in defining different solution concepts depending on the properties the cost function \costfunction should have.
In the next subsections we formalize different solutions to cost function learning tasks.

\subsection{Turning All the Input Plans Optimal}
The first objective we turn our attention to is trying to find a common cost function under which all the input plans are optimal.
We refer to such solutions as Ideal Cost Functions.
\begin{definition}
    Given a \cfl task, we define an \textbf{ideal cost function} \icf that solves it as a common cost function \costfunction under which all the plans in \multipleplans are optimal. The quality of an ideal cost function is defined as follows: 
    \begin{equation}
        \small
        \min \sum_{\action \in \actions} \cost(\action),
        \label{eq:solution_quality2}
    \end{equation}
    \vspace{-0.3cm}
    \begin{equation}
        \small
        \text{s.t.} \sum_{i \in \multipletasks,\multipleplans} \textsc{is\_optimal}(\plan_i,\task_i=\langle F,A,I_i,G_i,C\rangle) = \abs{\multipleplans},
        \label{eq:solution_quality1}
    \end{equation}
    An \icf is optimal iff no other cost function $\costfunction^\prime$ yields a lower value in Equation~(\ref{eq:solution_quality2}) while satisfying Constraint~(\ref{eq:solution_quality1}).
\end{definition}

\begin{remark}\label{rem:zero-costs}
    Note that we choose non-zero action costs in the definition of a \cfl task 
    to avoid trivializing the problem, as a cost function $\costfunction$ assigning zero to all actions could otherwise be an \icf.
\end{remark}

Ideal Cost Functions are not guaranteed to exist for arbitrary \cfl tasks, since Constraint (\ref{eq:solution_quality1}) cannot always be satisfied.
The difficulty lies in the inter dependencies and potential conflicts between plans. 
In particular, this not only refers to the actions shared between the input plans, but also to the multiple alternative plans that an input plan can have, i.e., all the other sequence of actions that can achieve \goal from \init. 

\begin{remark}
    Note that since we are considering only simple plans, the set of alternative plans $\alternativeplans{\plan_i}$ that a plan $\plan_i \in \multipleplans$ can have is finite. 
\end{remark}

We can, now, show that there is no guarantee that an \icf solution exists for a \cfl task.

\begin{theorem}\label{thm:no-cost-function}
Given a \cfl task, it is not guaranteed that there exists an \icf solution. 
\end{theorem}
\begin{proof}
Given a $\cfl = \langle \multipletasks, \multipleplans \rangle$ where $\multipletasks =(\task_i, \task_j)$,
and $\multipleplans~=~(\plan_i,\plan_j)$, let us assume that an \icf solution exists. Let us also assume that there are alternative plans $\alternativeplan_i \in \alternativeplans{\plan_i}$ and $\alternativeplan_j \in \alternativeplans{\plan_j}$ for $\plan_i$ and $\plan_j$ respectively, such that $\alternativeplan_i\subset\plan_j$ and $\alternativeplan_j\subset\plan_i$. 
Since $\plan_i$ is optimal, from Definition~\ref{def:optimal_plan} we have:
\begin{equation}\label{eq:ti-ai}
    \small
    \cost(\plan_i) \leq \cost(\alternativeplan_i)
\end{equation}

From the assumption $\alternativeplan_i \subset \plan_j$, there is at least one more action in $\plan_j$ than in $\alternativeplan_i$. Moreover, since the minimum cost for each action is $1$, then:
\begin{equation}\label{eq:ai-tj}
    \small
    \cost(\alternativeplan_i)< \cost(\plan_j)
\end{equation}
Since $\plan_j$ is also optimal, again by Definition~\ref{def:optimal_plan}: 
\begin{equation}\label{eq:tj-aj}
    \small
    \cost(\plan_j)\leq\cost(\alternativeplan_j)
\end{equation}
From~(\ref{eq:ti-ai}), (\ref{eq:ai-tj}) and (\ref{eq:tj-aj}) we have
\begin{equation}\label{eq:all}
    \small
    \cost(\plan_i)\leq\cost(\alternativeplan_i) < \cost(\plan_j)\leq\cost(\alternativeplan_j)
\end{equation}

However, from $\alternativeplan_j \subset \plan_i$, as the case above, there are at least one more action in $\plan_i$, we have
\begin{equation}\label{eq:aj-ti}
    \small
\cost(\alternativeplan_j) < \cost(\plan_i)
\end{equation}

From~(\ref{eq:all}) and (\ref{eq:aj-ti}) we get a contradiction from the assumption that there exists an \icf.
\end{proof}

Observe that the theorem would not hold if zero were allowed as the minimum cost for an action (see Remark~\ref{rem:zero-costs}). 
We illustrate this result by the following example:

\begin{example}
Let $\multipletasks =(\task_1, \task_2)$, 
and let $\multipleplans~=~(\plan_1,\plan_2)$ where $\plan_1=  [\addaction{move}{A}{C}, \addaction{move}{C}{B}]$ and $\plan_2=[\addaction{move}{A}{B}, \addaction{move}{B}{C}]$. These plans are displayed in the Figure below, where \textcolor{orange}{$\plan_1$} is represented by the \textcolor{orange}{orange} arrows, and \textcolor{blue}{$\plan_2$} by the \textcolor{blue}{blue} ones. 
\vspace{-2mm}
\begin{center}
\scalebox{0.55}{
\begin{tikzpicture}
\coordinate (A) at (0,0);
\coordinate (B) at (4,0);
\coordinate (C) at (2,3);

\draw (A) -- (B) -- (C) -- cycle;

\draw[blue, thick, ->] (A) -- ++(0.13,0.13) -- ++(3.6,0);
\draw[blue, thick, ->] ($(B)!0.07!(C)+ (-0.09,-0.09)$) -- ($(C)!0.07!(B)+ (-0.09,-0.09)$);

\draw[orange, thick, ->] (A) -- ++(-0.1,0.1) -- ++(2,3);
\draw[orange, thick, ->] ($(C)!0.0!(B)+ (0.09,0.09)$) -- ($(B)!0.07!(C)+ (0.09,0.09)$);

\draw[fill=white] (A) circle [radius=0.3] node {A};
\draw[fill=white] (B) circle [radius=0.3] node {B};
\draw[fill=white] (C) circle [radius=0.3] node {C};

\node at (0.8,2) {$1$};
\node at (3.1,2) {$1$};
\node at (2,-0.3) {$2$};
\end{tikzpicture}}
\end{center}

Let us assume that there is a \icf solution, i.e. there exists a cost function \costfunction under which all plans in \multipleplans are optimal.

Suppose the cost function \costfunction assigns the minimum costs for the actions that formed \textcolor{orange}{$\plan_1$}: $\cost\addaction{move}{A}{C} = 1$ and $\cost\addaction{move}{C}{B} = 1$. Then, we have $\cost(\textcolor{orange}{\plan_1}) = 2$. Since \textcolor{orange}{$\plan_1$} is optimal, the cost function \costfunction has to assign greater or equal costs to its alternative plans. Thus, $\cost\addaction{move}{A}{B} = 2$.

On the other hand, since we assume \costfunction exists, then \textcolor{blue}{$\plan_2$} is also optimal. The cost function \costfunction assigns the minimum costs to its actions, however, we already have $\cost\addaction{move}{A}{B}~=~2$. We then assign $\cost\addaction{move}{B}{C} = 1$. Therefore, we have $\cost(\textcolor{blue}{\plan_2}) = 3$. But there exists an alternative plan for \textcolor{blue}{$\plan_2$}, which is: $\cost\addaction{move}{A}{C} = 1$. Then, \textcolor{blue}{$\plan_2$} is not optimal. A similar conclusion is reached if we start with \textcolor{orange}{$\plan_1$}.
Thus, there is no \icf solution 
that guarantees all plans in \multipleplans are optimal.
\end{example}

\begin{remark}
    \label{remark_actions}
    Observe that the optimality of a plan $\plan_i \in \multipleplans$ only depends on the costs of the actions $(\action_i,\ldots, \action_n) \in A$ occurring in $\plan_i$ or in $\alternativeplans{\plan_i}$. The costs assigned to the rest of the actions do not affect $\plan_i$'s optimality.
\end{remark}

\subsection{Maximizing the Optimal Input Plans}
Given that \icfs 
are not guaranteed to exist for every \cfl task, we now relax the solution concept and focus on cost functions that maximize the number of plans turned optimal. We refer to such solutions as Maximal Cost Functions.

\begin{definition}
    Given a \cfl task, we define a \textbf{maximal cost function} \mcf that solves it as a common cost function \costfunction under which a maximum number of plans in \multipleplans are optimal. We formally establish the quality of a maximal cost function as follows:
    \begin{equation}
        \small
        \max \sum_{i \in \multipletasks,\multipleplans} \textsc{is\_optimal}(\plan_i,\task_i=\langle F,A,I_i,G_i,C\rangle),
        \label{eq:c_quality1}
    \end{equation}
    \vspace{-0.5cm}
    \begin{equation}
        \small
        \min \sum_{\action \in \actions} \cost(\action),
        \label{eq:c_quality2}
    \end{equation}
    A maximal cost function \mcf is optimal iff no other cost function $\costfunction^\prime$ yields a higher value in Equation~(\ref{eq:c_quality1}), and, if tied, a lower value in Equation~(\ref{eq:c_quality2}). 
    \label{def:mcfl}
\end{definition}

While \icfs are not guaranteed to exist, it is easy to show that there is always a \mcf that solves a \cfl task.
This is because even a cost function under which none of the plans \multipleplans are optimal would be a valid \mcf solution.
One might think that we can go one step further and ensure that there is a trivial cost function \costfunction guaranteeing that at least one of the plans in \multipleplans will be optimal in the resulting model. This trivial cost function would consist on assigning a cost $k$ to all the actions in one of the plans $\plan_i \in \multipleplans$, and a cost $k+|\reachablestates|$ to the rest of the actions in \actions. Unfortunately, this is not always the case. 
In particular, if the input plan \plan contains \emph{redundant actions}~\cite{nebel1997ignoring,salerno2023eliminating}, i.e., actions that can be removed without invalidating the plan, then it is not possible to turn \plan optimal.
Example~\ref{ex:redundant} illustrates this case.

\begin{example}
\label{ex:redundant}
Let $\multipletasks = (\task_1$) represented in the Figure below, where the initial state is displayed by the set of blocks on the left, and the goal state by the blocks on the right. Let $\multipleplans = (\plan_1$) be the plan that solves $\task_1$, where $\plan_1 = [\mathsf{unstack}(D,C), \mathsf{putdown}(D),\mathsf{unstack}(B, A), \mathsf{putdown}(B),\break \mathsf{pickup}(A), \mathsf{stack}(A,B)]$.

\begin{center}
\scalebox{0.7}{
\begin{tikzpicture}
    \draw[line width=0.5mm] (-0.5,0)--(3,0);
    
    \draw (0,0) rectangle (1,1) node[pos=.5] {C};
    \draw (0,1) rectangle (1,2) node[pos=.5] {D}; 
    \draw (1.5,0) rectangle (2.5,1) node[pos=.5] {A};
    \draw (1.5,1) rectangle (2.5,2) node[pos=.5] {B};
    
    \draw[->,line width=0.2mm] (3.2,1)--(4.2,1);
    
    \draw[line width=0.5mm] (4.7,0)--(6.3,0);
    \draw (5,0) rectangle (6,1) node[pos=.5] {B};
    \draw (5,1) rectangle (6,2) node[pos=.5] {A};
\end{tikzpicture}}
\end{center}
Let $\alternativeplan_1  = [(\mathsf{unstack}(B, A), \mathsf{putdown}(B), \mathsf{pickup}(A),\break \mathsf{stack}(A,B))]$ be an alternative plan. For $\plan_1$ to be optimal, its action costs must be lower than or equal to those of $\alternativeplan_1$. However, this is not possible because in $\plan_1$ there are redundant actions and there is no other action in $\alternativeplan_1$ to which we can assign a higher cost, thereby making the total cost of $\alternativeplan_1$ lower than that of $\plan_1$. Therefore, there is no cost function \costfunction such that $\plan_1$ is optimal.
\end{example}

Although in some extreme scenarios no input plan can be turned optimal, this is not the general case as we will see. 

\subsection{Making Input the Only Optimal Plans}
Up to now we have explored the problem of turning a set of input plans optimal.
We have showed that this task does not always have a solution when we want to make all the plans in \multipleplans optimal (\icf), and therefore, we focused on maximizing the number of plans that are optimal under the resulting cost function (\mcf).
These solutions still allow for the existence of other alternative plans with the same cost.
In some cases, we might be interested in a more restrictive setting, where \emph{only} the input plans are optimal.
In other words, there is no optimal plan $\plan_i$ such that $\plan_i \not\in \multipleplans$.
This can be the case of applications where we are interested in generating optimal plans that perfectly align with the user preferences, preventing the model from generating optimal plans outside of the observed behavior.
For this new solution we are still focusing on maximizing the number of plans in \multipleplans that are optimal, but now we want them to be the only optimal ones. We formally define this new solution as follows:

\begin{definition}[Strict Cost Function]\label{def:scfl}
    Given a \cfl task, we define a strict cost function \scf as a common cost function \costfunction under which a maximum number of plans in \multipleplans are the only optimal plans.
\end{definition}

The quality of a strict cost function solving a \cfl task is defined similarly to \mcf (Definition~\ref{def:mcfl}), but differs in Definition~\ref{def:optimal_plan}, which defines an optimal plan (used in the boolean function $\textsc{is\_optimal}(\plan,\task)$ in Equation (\ref{eq:c_quality1})).
The condition for a plan's cost being lower than or equal to($\leq$) its alternatives  is replaced by being strictly lower ($<$) in the strict approach. 
Similarly to \mcf, there is always a \scf that solves a \cfl task, and we cannot guarantee the existence of an ideal cost function for \scf.
In particular, Theorem~\ref{thm:no-cost-function} also applies to the \scf solution, with the exception that the definition of an optimal plan has changed as specified earlier.

\subsection{Adapting an Existing Cost Function}
In practice, we may have an approximate cost function \initcostfunction that we would like to refine with observed plans, rather than learning it from scratch as previously focused. In other words, the cost function \costfunction in the sequence of planning task \multipletasks is not empty as in Definition~\ref{def:cfl}. We then introduce a new task with initial costs as follows:

\begin{definition}\label{def:cflc}
    A \textbf{cost function refinement task} is a tuple $\cflc = \langle \multipletasks, \multipleplans \rangle$ where:
    \begin{itemize}
        \item $\multipletasks=(\task_1, \ldots, \task_n)$ is a sequence of planning tasks that share \fluents, \actions, and $\costfunction=\initcostfunction$.
        \item $\multipleplans=(\plan_1,\ldots,\plan_n)$ is a corresponding sequence of simple plans that solves \multipletasks.
    \end{itemize}
    The solution to a \cflc task is a common cost function \costfunction.
\end{definition}

We denote with \mcfc and \scfc when we have \mcf and \scf solutions respectively, but for a \cflc task.
We then formally redefine the quality of \mcfc and \scfc solutions by slightly modifying Definition~\ref{def:mcfl}.
In this case, we change the secondary objective of minimizing the sum of actions' costs (Equation (\ref{eq:c_quality2})) to minimizing the difference between the solution cost function \costfunction and the approximate cost function received as input \initcostfunction (Equation (\ref{eq:initc-quality2}) below):
\begin{equation}
        \small
        \min \sum_{\action \in \actions} \abs{\cost(\action)- \initcost(\action)}
        \label{eq:initc-quality2}
    \end{equation}
where $\initcost(\action)$ refers to the cost of each action given by the initial cost function \initcostfunction.
As before, it is easy to see that \mcfc (\scfc) solutions have the same properties as their \mcf (\scf) counterparts, i.e., there is always a cost function \costfunction that solves \cflc task, but we cannot guarantee the existence of a \costfunction that makes all the plans in \multipleplans optimal.

\subsection{Summary}
Let us summarize the different tasks we have presented so far (and their solutions) with the example illustrated in Table~\ref{tab:example-cost}.
The first row of the table shows the \cfl and \cflc tasks containing $6$ states labeled with letters.
The actions, depicted with edges, consist on moving between two connected states.
There are two input plans $\multipleplans=(\plan_1,\plan_2)$ such that \textcolor{teal}{$\plan_1$} $= [\addaction{move}{A}{B}, \addaction{move}{B}{D}]$ and \textcolor{purple}{$\plan_2$}$= [\addaction{move}{A}{C}, \addaction{move}{C}{E},\addaction{move}{E}{F}]$.
Colored states indicate they are an initial or goal state for one of the input plans.

The next rows depict optimal solutions for the two tasks.
The \mcf solution ensures both input plans are optimal, while assigning the minimum cost to each action.  
The \scf solution guarantees that the two input plans are the only optimal alternatives under the returned cost function \costfunction. 
That solution would also turn all the plans optimal under the \mcf definition, but would be suboptimal as it has a higher sum of action's costs.
This is because the cost of $\addaction{move}{C}{D}$ and $\addaction{move}{D}{F}$ need to be increased from $1$ to $2$ in order to force that there are no optimal plans outside \multipleplans.
On the right side, the approximate cost function \initcostfunction is minimally modified to guarantee that the input plans are optimal. 
For \mcfc, this is achieved by reducing by one the cost of $\addaction{move}{A}{C}$ and $\addaction{move}{E}{F}$.
For \scfc, we need to decrease the cost of $\addaction{move}{C}{E}$ and $\addaction{move}{E}{F}$, and increase the cost of $\addaction{move}{D}{F}$ to $2$ so as to force the stricter \scfc requirement.

\begin{table*}[t]
    \centering
    \begin{tabular}{|M{2cm}|M{4.3cm}||M{4.3cm}|M{2cm}|}
        \hline
        \cfl task
         &
         
        \scalebox{0.5}{\begin{tikzpicture}
            \node[circle, draw, minimum size=0.6cm, path picture={\fill[teal!70] (path picture bounding box.south west) rectangle (path picture bounding box.north); \fill[purple!70] (path picture bounding box.south) rectangle (path picture bounding box.north east);}] (A) at (0,0) {A};
            \node[circle, draw, minimum size=0.6cm] (B) at (0,-2) {B};
            \node[circle, draw, minimum size=0.6cm] (C) at (2,0) {C};
            \node[fill=teal!70,circle, draw, minimum size=0.6cm] (D) at (2,-2) {D};
            \node[circle, draw, minimum size=0.6cm] (E) at (4,0) {E};
            \node[fill=purple!70, circle, draw, minimum size=0.6cm] (F) at (4,-2) {F};
            
            \draw (A)--(B)--(D)--(C)--cycle;
            \draw (C)--(E)--(F)--(D)--cycle;

            \draw[purple, thick] (A)--(C);
            \draw[purple, thick] (C)--(E);
            \draw[purple, thick] (E)--(F);
            \draw[teal, thick] (B)--(A);
            \draw[teal, thick] (D)--(B);
        \end{tikzpicture} }
        &
        \scalebox{0.5}{\begin{tikzpicture}
            \node[circle, draw, minimum size=0.6cm, path picture={\fill[teal!70] (path picture bounding box.south west) rectangle (path picture bounding box.north); \fill[purple!70] (path picture bounding box.south) rectangle (path picture bounding box.north east);}] (A) at (0,0) {A};
            \node[circle, draw, minimum size=0.6cm] (B) at (0,-2) {B};
            \node[circle, draw, minimum size=0.6cm] (C) at (2,0) {C};
            \node[fill=teal!70,circle, draw, minimum size=0.6cm] (D) at (2,-2) {D};
            \node[circle, draw, minimum size=0.6cm] (E) at (4,0) {E};
            \node[fill=purple!70,circle, draw, minimum size=0.6cm] (F) at (4,-2) {F};
            
            \draw (A)--(B)--(D)--(C)--cycle;
            \draw (C)--(E)--(F)--(D)--cycle;

            \draw[purple, thick] (A)--(C);
            \draw[purple, thick] (C)--(E);
            \draw[purple, thick] (E)--(F);
            \draw[teal, thick] (B)--(A);
            \draw[teal, thick] (D)--(B);

            \node at (1,0.2) {2};
            \node at (1,-2.2) {1};
            \node at (-0.2,-1) {2};
            \node at (3,0.2) {2};
            \node at (3,-2.2) {1};
            \node at (4.2,-1) {2};
            \node at (2.2,-1) {2}; 
        \end{tikzpicture} }
        &
        \cflc task
        \\
        \hline
        \hline
        \mcf 
        &
        \scalebox{0.5}{\begin{tikzpicture}
            \node[circle, draw, minimum size=0.6cm] (A) at (0,0) {A};
            \node[circle, draw, minimum size=0.6cm] (B) at (0,-2) {B};
            \node[circle, draw, minimum size=0.6cm] (C) at (2,0) {C};
            \node[circle, draw, minimum size=0.6cm] (D) at (2,-2) {D};
            \node[circle, draw, minimum size=0.6cm] (E) at (4,0) {E};
            \node[circle, draw, minimum size=0.6cm] (F) at (4,-2) {F};

            \draw (A)--(B)--(D)--(C)--cycle;
            \draw (C)--(E)--(F)--(D)--cycle;
    
            \draw (A)--(C);
            \draw (C)--(E);
            \draw (B)--(A);
            \draw (D)--(B);
    
            \node at (1,0.2) {1};
            \node at (1,-2.2) {1};
            \node at (-0.2,-1) {1};
            \node at (3,0.2) {1};
            \node at (3,-2.2) {1};
            \node at (4.2,-1) {1};
            \node at (2.2,-1) {1}; 
        \end{tikzpicture} } 
         &
         \scalebox{0.5}{\begin{tikzpicture}
            \node[circle, draw, minimum size=0.6cm] (A) at (0,0) {A};
            \node[circle, draw, minimum size=0.6cm] (B) at (0,-2) {B};
            \node[circle, draw, minimum size=0.6cm] (C) at (2,0) {C};
            \node[circle, draw, minimum size=0.6cm] (D) at (2,-2) {D};
            \node[circle, draw, minimum size=0.6cm] (E) at (4,0) {E};
            \node[circle, draw, minimum size=0.6cm] (F) at (4,-2) {F};
            
            \draw (A)--(B)--(D)--(C)--cycle;
            \draw (C)--(E)--(F)--(D)--cycle;
    
            \draw (A)--(C);
            \draw (C)--(E);
            \draw (B)--(A);
            \draw (D)--(B);
    
            \node at (1,0.2) {1};
            \node at (1,-2.2) {1};
            \node at (-0.2,-1) {2};
            \node at (3,0.2) {2};
            \node at (3,-2.2) {1};
            \node at (4.2,-1) {1};
            \node at (2.2,-1) {2}; 
        \end{tikzpicture} } 
         &
         \mcfc 
         \\
         \hline
         \begin{tikzpicture}
            \node at (1,0.2) {\scf };
         \end{tikzpicture}
          &
        \scalebox{0.5}{\begin{tikzpicture}
            \node[circle, draw, minimum size=0.6cm] (A) at (0,0) {A};
            \node[circle, draw, minimum size=0.6cm] (B) at (0,-2) {B};
            \node[circle, draw, minimum size=0.6cm] (C) at (2,0) {C};
            \node[circle, draw, minimum size=0.6cm] (D) at (2,-2) {D};
            \node[circle, draw, minimum size=0.6cm] (E) at (4,0) {E};
            \node[circle, draw, minimum size=0.6cm] (F) at (4,-2) {F};
            
            \draw (A)--(B)--(D)--(C)--cycle;
            \draw (C)--(E)--(F)--(D)--cycle;
    
            \draw (A)--(C);
            \draw (C)--(E);
            \draw (B)--(A);
            \draw (D)--(B);
    
            \node at (1,0.2) {1};
            \node at (1,-2.2) {1};
            \node at (-0.2,-1) {1};
            \node at (3,0.2) {1};
            \node at (3,-2.2) {2};
            \node at (4.2,-1) {1};
            \node at (2.2,-1) {2}; 
        \end{tikzpicture} } 
         &
          \scalebox{0.5}{\begin{tikzpicture}
            \node[circle, draw, minimum size=0.6cm] (A) at (0,0) {A};
            \node[circle, draw, minimum size=0.6cm] (B) at (0,-2) {B};
            \node[circle, draw, minimum size=0.6cm] (C) at (2,0) {C};
            \node[circle, draw, minimum size=0.6cm] (D) at (2,-2) {D};
            \node[circle, draw, minimum size=0.6cm] (E) at (4,0) {E};
            \node[circle, draw, minimum size=0.6cm] (F) at (4,-2) {F};
            
            \draw (A)--(B)--(D)--(C)--cycle;
            \draw (C)--(E)--(F)--(D)--cycle;
    
            \draw (A)--(C);
            \draw (C)--(E);
            \draw (B)--(A);
            \draw (D)--(B);
    
            \node at (1,0.2) {2};
            \node at (1,-2.2) {1};
            \node at (-0.2,-1) {2};
            \node at (3,0.2) {1};
            \node at (3,-2.2) {2};
            \node at (4.2,-1) {1};
            \node at (2.2,-1) {2}; 
        \end{tikzpicture} }
        &
        \scfc 
        \\
        \hline
    \end{tabular}
    \vspace{2mm}
    \caption{Optimal solutions \mcf, \scf, \mcfc and \scfc for \cfl and \cflc tasks. Letters represent the states of the planning tasks, with edges representing the actions. Input plans \textcolor{teal}{$\plan_1$} and \textcolor{purple}{$\plan_2$} are depicted in teal and purple, respectively.
    }
    \vspace{4mm}
    \label{tab:example-cost}
\end{table*}

%% file: sections/04b-algorithm.tex
\section{Solving Cost Function Learning Tasks}
Algorithm 1 describes \optimalapproach{k}, an algorithm to learn action's costs from input plans.
It receives the cost function learning (refinement) task to solve ($T$), the desired solution ($S$), and a parameter $k$ that determines the number of alternative plans to be computed for each plan $\plan_i \in \multipleplans$.
Higher $k$ values indicate a higher percentage of $\alternativeplans{\plan_i}$ is covered, with $k=\infty$ meaning that the whole set of alternative plans is computed.
\optimalapproach{k} first generates a Mixed-Integer Linear Program (MILP) to assign costs to actions, i.e., to compute the solution cost function \costfunction.
This MILP is shown below (Equations~\ref{milp:objective_function}-\ref{milp:z_domain}) to compute \mcf solutions for \cfl tasks.
MILPs for the other solutions and tasks are similar and can be found in the Appendix.

\input{algorithms/algorithm}
\vspace{-5mm}
\input{algorithms/milp}

We have three sets of decision variables.
The first, $x_\plan$, are binary decision variables that will take a value of $1$ if plan \plan is optimal in the resulting domain model, and $0$ otherwise (Equation (\ref{milp:x_domain})).
The second, $z_{\plan,\alternativeplan}$ are binary decision variables that will take a value of $1$ if plan \plan has a lower or equal cost than alternative plan \alternativeplan, and $0$ otherwise (Equation (\ref{milp:z_domain})).
Finally, $y_\action$ are integer decision variables that assign a cost to each action $\action \in \relevantactions$.
This subset of actions $\relevantactions \subseteq \actions$ represents the actions which cost needs to be set in order to make the plans in \multipleplans optimal (see Remark~\ref{remark_actions}).
Constraint (\ref{milp:plan_vs_alternative_optimality_constraint}) enforces the value of the $z$ variables.
This is done by setting $M$ to a large number, forcing $z_{\plan,\alternativeplan}$ to be $1$ iff $c(\plan)=\sum_{\action \in \plan} y_\action \leq \sum_{\action \in \alternativeplan} y_\action=c(\alternativeplan)$.
Similarly, Constraint (\ref{milp:plan_optimality_constraint}) ensures that $x_\plan=1$ iff \plan has a lower or equal cost than the rest of its alternative plans. 

We aim to optimize the objective function described in Equation (\ref{milp:objective_function}), where we have two objectives.
The first objective, weighted by $\omega_1$, aims to maximize the number of optimal plans.
The second objective, weighted by $\omega_2$, aims to minimize the total cost of the actions in the learned model.
As described in Algorithm 1, \optimalapproach{k} will first try to maximize the number of plans that can be turned optimal by setting the weights to $\omega_1=1$ and $\omega_2=0$ (line 2).
Solving the MILP with these weights will give us a cost function $C$ that makes $Q$ plans optimal.
Then, \optimalapproach{k} generates a new MILP with a new constraint (line 5), enforcing that the new solution has to turn exactly $Q$ plans optimal.
This second MILP is solved with $\omega_1=0$ and $\omega_2=1$ in order to find the optimal cost function $C$, which makes $Q$ plans optimal and minimizes the sum of action's costs.
After that, \optimalapproach{k} updates the cost function \costfunction (line 7) by assigning a cost to the actions the MILP does not reason about, i.e., the actions $\actions \setminus \relevantactions$ that do not affect the optimality of the plans in \multipleplans. 
For example, in the case of \mcf solutions, this function assigns the minimum cost ($1$) to these actions.
This updated cost function \costfunction is finally returned by \optimalapproach{k} as the solution to the cost function learning task.
\optimalapproach{k}'s optimality proof can be found in the Appendix.

%% file: algorithms/algorithm.tex
\begin{algorithm}
\small
\caption{\optimalapproach{k}}
\label{alg:lacfipk}
\renewcommand{\algorithmicrequire}{\textbf{Input:}}
\renewcommand{\algorithmicensure}{\textbf{Output:}}
    \begin{algorithmic}[1]
    \Require Task $T$, Solution concept $S$, \# of alternatives $k$
    \Ensure Cost function \costfunction
    
    \State $\relevantactions \gets \textsc{getRelevantActions}(\actions, \multipleplans)$
    
    \State $\textsc{MILP} \gets \textsc{generateMILP}(T, S, \relevantactions, k)$
    
    \State$C, Q \gets \textsc{solve}(\textsc{MILP}_1, \omega_1=1, \omega_2=0)$
    
   \State $\mbox{constraint} \gets \sum_{\plan \in \multipleplans} x_\plan = Q$
    
    \State $\textsc{MILP} \gets \textsc{addConstraint}(\textsc{MILP}, \mbox{constraint})$
    
    \State $C, Q \gets \textsc{solve}(\textsc{MILP}, \omega_1=0, \omega_2=1)$
    
    \State $C \gets \textsc{updateCostFunction}(\costfunction, \actions, \relevantactions)$
    
    \State \Return{$C$}
    \end{algorithmic}
\end{algorithm}

%% file: algorithms/milp.tex
\begin{equation}
    \small
    \label{milp:objective_function}
    \text{maximize} \quad \omega_1 \displaystyle\sum\limits_{\plan \in \multipleplans} x_{\plan} - \omega_2 \displaystyle\sum\limits_{\action \in \relevantactions} y_{\action}
\end{equation}
\begin{equation}
    \small
    \label{milp:plan_vs_alternative_optimality_constraint}
    \text{s.t.} \displaystyle\sum\limits_{\action \in \plan}  y_{\action}  \leq \displaystyle\sum_{\action \in \alternativeplan} y_{\action} + M(1-z_{\plan,\alternativeplan}),  \plan \in \multipleplans, \alternativeplan \in \alternativeplans{\plan}
\end{equation}
\begin{equation}
    \small
    \label{milp:plan_optimality_constraint}
    |\alternativeplans{\plan}| - \displaystyle\sum_{\alternativeplan \in \alternativeplans{\plan}} z_{\plan,\alternativeplan}   \leq M(1-x_{\plan}),  \plan \in \multipleplans
\end{equation}
\begin{equation}
    \small
    \label{milp:x_domain}
    x_{\plan}            \in \{0,1\}, \plan \in \multipleplans
\end{equation}
\begin{equation}
    \small
    \label{milp:y_domain}
    y_{\action}        \geq 1, \action \in \relevantactions
\end{equation}
\begin{equation}
    \small
    \label{milp:z_domain}
     z_{\plan,\alternativeplan}        \in \{0,1\}, \plan \in \multipleplans, \alternativeplan \in \alternativeplans{\plan}
\end{equation}

%% file: sections/05-evaluation.tex
\section{Evaluation}
We run three different experiments:
\begin{enumerate}
    \item \textbf{Alignment}. First, we test if \optimalapproach{k} can accurately learn cost functions that align with observed behavior, i.e., the input plans. We conduct experiments in a controlled \textsc{grid} navigation domain, where an agent moves in four cardinal directions to reach a target cell. We generate plans with features like avoiding or preferring certain areas and qualitatively assess if the cost functions from \optimalapproach{k} correctly capture this implicit knowledge.

    \item \textbf{Suboptimality}. \optimalapproach{k} is  guaranteed to maximize the number of plans turned optimal when $k=\infty$, a condition that is impractical for most planning tasks. In this experiment we examine how the algorithm's performance deviates from the best possible value as we incrementally increase $k$ in small \textsc{grid} \cfl tasks.

    \item \textbf{Performance and Scalability}. The algorithm's performance is influenced by two main factors: (i) the value of $k$, which represents the number of alternative plans we wish to consider, and (ii) the \cfl size, which refers to the number and size of the input tasks and plans. In this experiment, we assess how \optimalapproach{k} scales as we increase both parameters across various planning domains.
\end{enumerate}

We only report here results when computing \mcf solutions for \cfl tasks due to space constraints.
Results for the other tasks and solution concepts can be found in the Appendix.

\textbf{Reproducibility}. \optimalapproach{k} uses the unordered top-k configuration of \textsc{symk}~\cite{speck2020symbolic} to compute the set of $k$ alternative plans, and solves the resulting MILPs using the CBC solver~\cite{forrest2005cbc}.
We compare \optimalapproach{k} against \baseline, an  algorithm that assigns either (i) the minimum cost ($1$) to all the actions when solving \cfl tasks; or (ii) the cost prescribed by the approximate cost function when solving \cflc.
Both algorithms leverage \textsc{Fast Downward}~\cite{helmert2006fast} translator to get the grounded actions of a planning task.
Experiments were run on AMD EPYC 7R13 CPUs @ 3.6GHz with an 8GB memory limit and a total time limit of $1800$s per algorithm and \cfl task.

\subsection{Alignment}
\begin{figure}
\vspace{-3mm}
    \centering
    \includegraphics[width=0.71\linewidth]{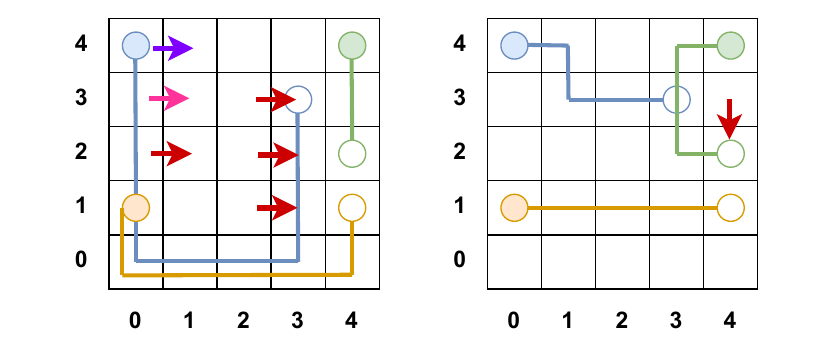}
    \vspace{1mm}
    \caption{Two \cfl tasks in the \textsc{grid} domain, each with $3$ input problems and plans shown in different colors. Filled dots are starting points ($\init$), empty dots are goals ($\goal$), and lines represent plans ($\plan$). The solution assigns costs to actions, with arrows indicating costs above $1$: red for $3$, pink for $5$, and purple for $7$.}
    \vspace{4mm}
    \label{fig:alignment}
\end{figure}
We generate the two \cfl tasks shown in Figure~\ref{fig:alignment}.
The \textsc{grid} domain is a simplified real-world navigation domain, where application owners can gather data on how drivers reach their destinations to infer their preferences. 
In the task on the left, plans aim to avoid the center and top parts of the map, while in the task on the right, plans strive to avoid the outer part of the grid as much as possible. 
This behavior may occur due to traffic jams or construction in these city areas, prompting drivers to choose longer routes to avoid congestion. 
By turning these observed plans optimal, application owners can develop/adjust a model that suggests more useful routes to users.

\optimalapproach{k} effectively turns these input plans optimal.
In the left \cfl task, the \mcf solution provided by \optimalapproach{k} assigns a high cost to moving right through the grid center. 
This adjustment ensures that the blue and orange plans, in particular, become the best options for reaching their final destinations.
For instance, under the returned cost function, 
the blue plan incurs a cost of $10$. This is the same cost as a plan 
that moves right three times and then down, as the first action has a cost of $7$. 
In the right \cfl task, the input plans are shorter and interfere less with each other, so \optimalapproach{k} only needs to increase the cost of one action to ensure that the green plan is optimal.

This example demonstrates how \optimalapproach{k} can create cost functions $C$ that align with the implicit preferences of the input plans. 
In this case, the input plans follow a consistent pattern, enabling \optimalapproach{k} to identify a cost function where all plans are optimal. However, real-world situations often involve input plans from agents with diverse preferences, leading to inconsistencies. To simulate this diversity, we will generate random input plans for the rest of the evaluation.

\subsection{Suboptimality}

We now examine how \optimalapproach{k}'s performance deviates from the best possible value as we incrementally increase $k$ in small $5\times5$ \textsc{grid} instances like the one shown in Figure~\ref{fig:alignment}.
We generate $10$ different problems by varying $\init$ and $\goal$.
Then, we use \textsc{symk} to compute $100$ plans for each planning task $\task_i$.
This represents our pool of $10 \times 100 = 1000$ tuples $\langle \task_i, \plan_i \rangle$. 
We generate $10$ \cfl tasks by randomly selecting $10$ of these tuples.
We evaluate \textsc{baseline} and \optimalapproach{k} using four different input values: $k=10^1, 10^2, 10^3$ and $\infty$.

\input{tikz_figures/table-suboptimality}
Table 2 presents the results of this experiment.
The rows indicate the given algorithm, and the cells display the mean and standard deviation of ratio of input plans made optimal.
This ratio is computed by using the cost function \costfunction returned by each algorithm and using it to solve each of the planning tasks $\task_i \in \multipletasks$.
In the case of \mcf, we verify that optimally solving $\task_i$ with \costfunction yields the same cost as the sum of action's costs of the input plan, i.e., $\sum_{a \in \pi_i} c(a)$.
When this holds, input plan $\pi_i$ optimally solves task $\task_i$ and we annotate a $1$.
Otherwise, we annotate a $0$ meaning that \costfunction does not make $\plan_i$ optimal.
Cells are color-coded to indicate the ratio of optimal plans achieved: lighter (darker) colors represent a lower (higher) ratio of plans turned optimal.
As expected, \optimalapproach{k} achieves the best results when $k=\infty$.
This represents the maximum possible ratio of input plans turned optimal, but it requires computing all alternative plans for each planning task. 
Even in these relatively simple \textsc{grid} tasks, some planning tasks may have around $8,000$ simple plans, taking several hundred seconds to compute. 
Conversely, lower values of $k$ still yield very good results, often close to optimal. 
Specifically, with $k=10^3$, we achieve the same results as with $k=\infty$, and with $k=10^2$, the score is only slightly worse in one \cfl task.

\subsection{Performance and Scalability}
\input{tikz_figures/table-experiments}

In the previous experiment we used small instances in the \textsc{grid} domain where \optimalapproach{k} was able to return optimal cost functions $C$ within the time and memory bounds in order to understand how the suboptimality of the algorithm behaves as we increase $k$.
Now we want to understand \optimalapproach{k}'s performance and scalability as we consider larger \cfl tasks across various domains.
We ran experiments in four planning domains: \textsc{barman}, \textsc{openstacks}, \textsc{transport} and \textsc{grid}.
We chose them since (i) we wanted to get a representative yet small set of domains; and (ii) we required domains with existing problem generators that enable us to generate solvable tasks of a controllable size~\cite{seipp2022}.
For each domain, we fix the problem size (\fluents and \actions) and generate $50$ different problems by varying \init and \goal.
The problem sizes were chosen to allow computation of multiple alternatives in reasonable time.
For example, the \textsc{grid} size is $10\times10$, and \textsc{barman} tasks have $3$ ingredients, $3$ cocktails and $4$ shots.
Then, we use \textsc{symk}~\cite{speck2020symbolic} to compute $100$ simple plans for each planning task $\task_i$.
This represents our pool of $50\times100=5\,000$ tuples $\langle \task_i, \plan_i \rangle$ per each domain.
We generate \cfl tasks by randomly selecting $10^1$, $10^2$ or $10^3$ of these tuples from the pool.
For each \cfl task size we generate $10$ random problems, i.e., $10$ different sets of $\langle \task_i, \plan_i \rangle$ tuples, therefore having a total of $3\times 10 = 30$ \cfl tasks  per each domain.
The same tasks are transformed into \cflc tasks by using the cost function in the original planning task as \initcostfunction.

\paragraph{Performance.} Table 3 presents the results of this experiment. 
Each domain contains $10$ \cfl tasks of varying sizes: $10^1, 10^2$ and $10^3$, as displayed in the second row of the Table. These are solved by the following five algorithms: \textsc{baseline}, and \optimalapproach{k} with $k=10^1$, $k=10^2$, $k=10^3$, and $k=10^4$.
As in the suboptimality evaluation, cells display the ratio of plans made optimal for all problem instances commonly solved by at least one algorithm. 
Cells shaded in \textcolor{gray}{gray} indicate that the given algorithm failed to solve any of the $10$ \cfl tasks of that size. The remaining cells with values are color-coded to indicate the ratio of optimal plans achieved: lighter (darker) colors represent a lower (higher) ratio of optimal plans. 

 We identify two main trends. Firstly, \optimalapproach{k} is consistently better than the baseline, and its performance tends to improve as $k$ increases. This improvement is not necessarily monotonic, as we can see in \textsc{Grid} where $k=10^3$ obtains the best results. 
 This is expected, since with any $k$ value other than $\infty$, the MILP is not considering all the alternatives and might be leaving out some of the important ones, i.e., those that can affect the input's plan optimality.
 The ratio of plans turned optimal seems to saturate with low $k$ values, suggesting that in some domains we might not need to compute many alternatives to achieve good results. As the planning tasks grow, \optimalapproach{k} scales worse, with \cfl tasks in \textsc{barman} or \textsc{transport} where it cannot produce solutions within the time bound when $k>10^2$.
 Secondly, as we increase the size of the \cfl tasks, more conflicting problems with redundant actions can arise, reducing the ratio of plans that can be made optimal. For example, in \textsc{Barman}, we can see a gradient in color from left to right as the size of the \cfl tasks increases. In \textsc{Openstacks} and \textsc{Transport}, the problems are more complex, and \textsc{symk} fails to compute alternative plans within the time limit.

\paragraph{Scalability.}
Figure 2 illustrates the total execution time of each algorithm as we increase the \cfl task size.
Executions exceeding $1800$s are shown above the dashed line. 
\textsc{baseline}'s execution time remains constant as the \cfl size increases, being able to return cost functions in less than $10$s in all cases.
On the other hand, \optimalapproach{k} takes more time as more plans need to be turned optimal. For example, while it can solve all but two \cfl task of size $10^1$ in \textsc{grid}, $12$ time out when $|\cfl|=10^2$. 
Increasing $k$ is the factor that affects \optimalapproach{k} the most, as we can see in the linear (logarithmic) increase in the execution time regardless of the \cfl size and domain.
Finally, we analyzed how each component of \optimalapproach{k} affects the total running time.
For lower values of $k$, \textsc{symk} can compute the alternative plans in few seconds, and most of the running time is devoted to the MILP.
On the other hand, when $k\geq10^3$, computing the alternative plans takes most of the time, with problems where \optimalapproach{k} spends the $1800$s running \textsc{symk} with no available time to run the MILP.

\input{figures/execution_time}

%% file: tikz_figures/table-suboptimality.tex
\begin{table}[h]\label{tab:suboptimal}
    \centering
    \renewcommand{\tabcolsep}{0.06cm}
    \scalebox{0.82}{\begin{tabular}{|c|c|}
     \hline
      & Optimal Plans Ratio \\ \hline
      \textsc{baseline} & \cellcolor{YellowGreen!15}{$0.05 \pm 0.07$}\\ \hline
     $k = 10^1$ & \cellcolor{YellowGreen!15}{$0.07 \pm 0.11$} \\ \hline
      $k = 10^2$&\cellcolor{YellowGreen!45}{$0.53 \pm 0.24$}  \\ \hline
      $k = 10^3$ & \cellcolor{YellowGreen!65}{$0.55 \pm 0.24$}\\ \hline
      $k = \infty$&\cellcolor{YellowGreen!65}{$0.55 \pm 0.24$}\\ \hline
    \end{tabular}}
    \vspace{2mm}
    \caption{ Mean and standard deviation of the ratio of plans turned optimal when the different algorithms compute \mcf solutions in \textsc{grid}.}
    \vspace{3mm}
\end{table}

%% file: tikz_figures/table-experiments.tex
\begin{table*}[h]\label{tab:ratio-optimal}
    \centering
    \renewcommand{\tabcolsep}{0.06cm}
    \scalebox{0.82}{\begin{tabular}{|c|c|c|c|c|c|c|c|c|c|c|c|c|}
     \hline
      & \multicolumn{3}{c|}{\textsc{grid}} & \multicolumn{3}{c|}{\textsc{barman}} & \multicolumn{3}{c|}{\textsc{openstacks}} & \multicolumn{3}{c|}{\textsc{transport}} \\ \hline
      & $|\cfl|=10^1$ & $|\cfl|=10^2$ & $|\cfl|=10^3$ & $|\cfl|=10$ & $|\cfl|=10^2$ & $|\cfl|=10^3$ & $|\cfl|=10^1$ & $|\cfl|=10^2$ & $|\cfl|=10^3$ & $|\cfl|=10^1$ & $|\cfl|=10^2$ & $|\cfl|=10^3$ \\ \hline
      \textsc{baseline} & \cellcolor{YellowGreen!35}{$0.27 \pm 0.10$}&\cellcolor{YellowGreen!55}{$0.46 \pm 0.00$} & \cellcolor{YellowGreen!45}{$0.37 \pm 0.05$} &\cellcolor{YellowGreen!25}{$0.17 \pm 0.08$} &\cellcolor{YellowGreen!25}{$0.18 \pm 0.30$} &\cellcolor{YellowGreen!15}{$0.06 \pm 0.04$} &\cellcolor{YellowGreen!7}{$0.03 \pm 0.05$} & \cellcolor{YellowGreen!7}{$0.01 \pm 0.01$} &\cellcolor{YellowGreen!7}{$0.01 \pm 0.00$}&\cellcolor{YellowGreen!15}{$0.10 \pm 0.07$} &\cellcolor{YellowGreen!7}{$0.02 \pm 0.00$} & \cellcolor{YellowGreen!7}{$0.03 \pm 0.00$} \\ \hline
     $k = 10^1$ & \cellcolor{YellowGreen!45}{$0.43 \pm 0.16$}&\cellcolor{YellowGreen!35}{$0.26 \pm 0.00$} &\cellcolor{YellowGreen!45}{$0.37 \pm 0.05$} & \cellcolor{YellowGreen}{$1.00 \pm 0.00$}&\cellcolor{YellowGreen!65}{$0.65 \pm 0.27$} &\cellcolor{YellowGreen!15}{$0.08 \pm 0.03$} &\cellcolor{YellowGreen!7}{$0.05 \pm 0.10$}&\cellcolor{YellowGreen!7}{$0.00 \pm 0.01$} & \cellcolor{YellowGreen!7}{$0.01 \pm 0.01$}&\cellcolor{YellowGreen!25}{$0.18 \pm 0.08$} &\cellcolor{YellowGreen!7}{$0.03 \pm 0.00$} & \cellcolor{YellowGreen!7}{$0.00 \pm 0.00$} \\ \hline
      $k = 10^2$&\cellcolor{YellowGreen!55}{$0.54 \pm 0.25$} &\cellcolor{YellowGreen!45}{$0.35 \pm 0.00$} &\cellcolor{YellowGreen!45}{$0.37 \pm 0.05$} & \cellcolor{YellowGreen}{$1.00 \pm 0.00$}&\cellcolor{YellowGreen!65}{$0.65 \pm 0.27$} &\cellcolor{YellowGreen!15}{$0.08 \pm 0.03$} &\cellcolor{YellowGreen!25}{$0.19 \pm 0.30$} &\cellcolor{YellowGreen!7}{$0.00 \pm 0.00$} & \cellcolor{YellowGreen!15}{$0.12 \pm 0.19$}&\cellcolor{YellowGreen!25}{$0.16 \pm 0.10$} & \cellcolor{YellowGreen!7}{$0.00 \pm 0.00$}&\cellcolor{YellowGreen!7}{$0.00 \pm 0.00$} \\ \hline
      $k = 10^3$ &\cellcolor{YellowGreen!85}{$0.81 \pm 0.12$} &\cellcolor{YellowGreen!55}{$0.49 \pm 0.00$} &\cellcolor{YellowGreen!45}{$0.37 \pm 0.05$} & \cellcolor{YellowGreen}{$1.00 \pm 0.00$}&\cellcolor{YellowGreen!65}{$0.65 \pm 0.27$} &\cellcolor{YellowGreen!15}{$0.08 \pm 0.03$} &\cellcolor{gray!70}{} &\cellcolor{gray!70}{}& \cellcolor{gray!70}{}&\cellcolor{gray!70}{} & \cellcolor{gray!70} & \cellcolor{gray!70}\\ \hline
      $k = 10^4$&\cellcolor{YellowGreen!85}{$0.78 \pm 0.15$} &\cellcolor{YellowGreen!55}{$0.49 \pm 0.00$} & \cellcolor{YellowGreen!45}{$0.37 \pm 0.05$}&\cellcolor{gray!70}{} &\cellcolor{gray!70}{} &\cellcolor{gray!70}{} &\cellcolor{gray!70}{} &\cellcolor{gray!70}{} &\cellcolor{gray!70}{} &\cellcolor{gray!70}{}& \cellcolor{gray!70}{} & \cellcolor{gray!70}\\ \hline
    \end{tabular}}
    \vspace{2mm}
    \caption{ Mean and standard deviation of the ratio of plans turned optimal when the different algorithms compute \mcf solutions.}
    \vspace{3mm}
\end{table*}

%% file: figures/execution_time.tex
\begin{figure*}
    \centering
    \begin{subfigure}[t]{0.26\textwidth}\centering
    \includegraphics[width=\textwidth]{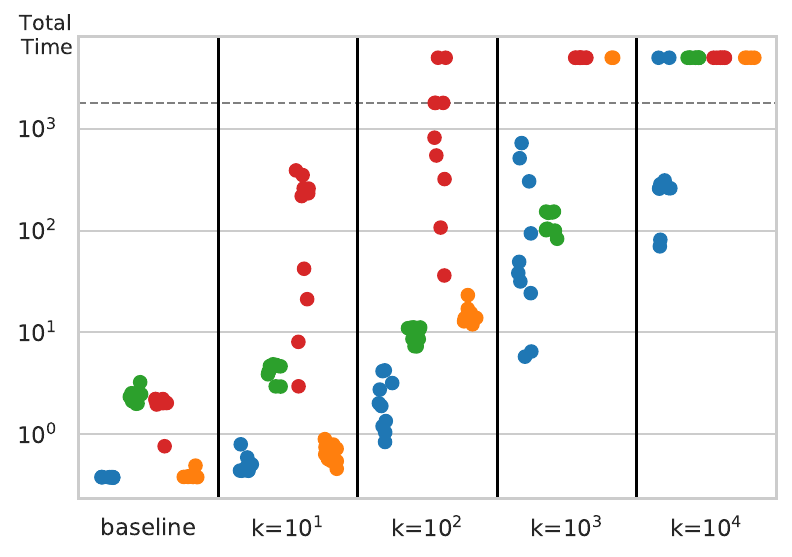}
    \caption{$|\cfl|=10^1$}
    \label{fig:cfl10}
    \end{subfigure}
    \hfill
    \begin{subfigure}[t]{0.26\textwidth}\centering
    \includegraphics[width=\textwidth]{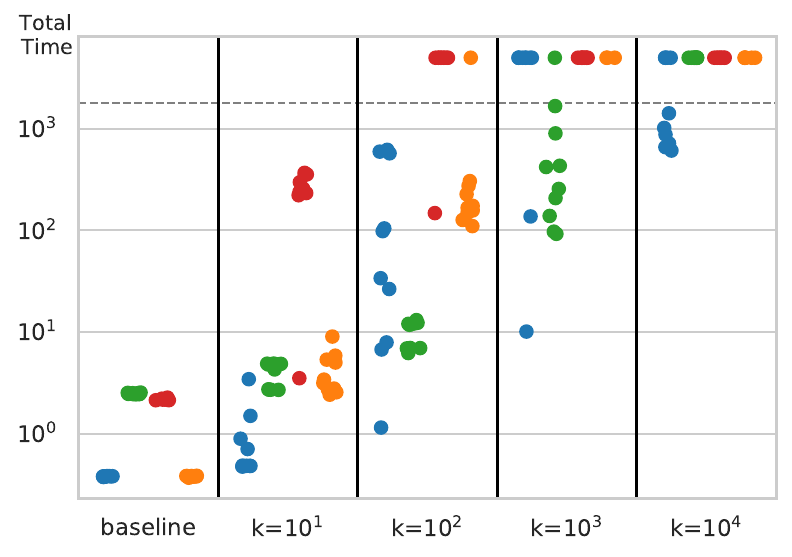}
    \caption{$|\cfl|=10^2$}
    \label{fig:cfl100}
    \end{subfigure}
    \hfill
    \begin{subfigure}[t]{0.26\textwidth}\centering
    \includegraphics[width=\textwidth]{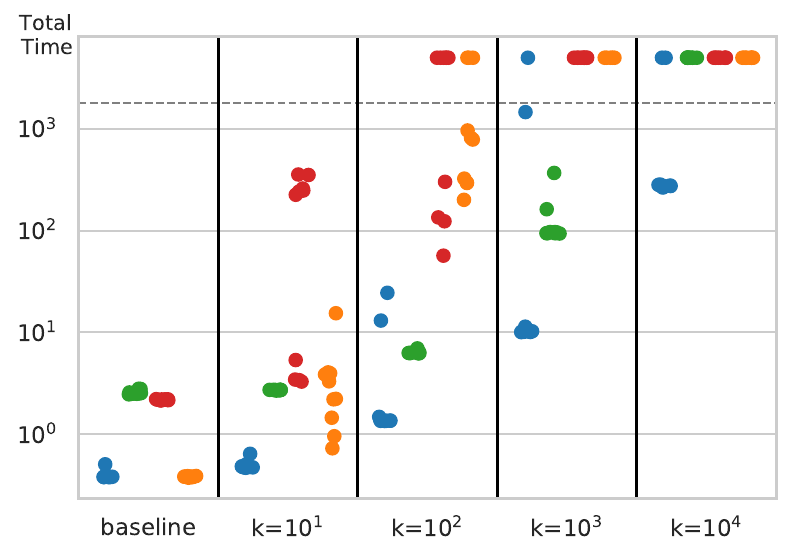}
    \caption{$|\cfl|=10^3$}
    \label{fig:cfl1000}
    \end{subfigure}
    \vspace{4mm}
    \caption{Total execution time for the four domains across different numbers of alternative plans $k$, with each plot corresponding to one of the CFL task sizes. 
    Domains are represented by different colors: \textcolor{blue}{\textsc{Grid}}, \textcolor{Green}{\textsc{Barman}}, \textcolor{orange}{\textsc{Openstacks}}, and \textcolor{red}{\textsc{Transport}}.}
    \vspace{4mm}
\end{figure*}

%% file: sections/06-related-work.tex
\section{Related Work}

\paragraph{Automated Planning.}
Most planning works on domain learning focus on acquiring the action's dynamics  given a set of input plans, often overlooking learning the action's cost model. 
\textsc{nlocm}~\cite{gregory2016domain} and \textsc{lvcp}~\cite{garrido2023learning} are two notable exceptions in the literature. 
Although differing in the input assumptions and guarantees, both approaches require each plan to be annotated with its total cost.
Then, they use constraint programming to assign costs to actions, ensuring their sum equals the total plan cost. 
Our work differs from them in three main aspects.
First, we focus on learning the action's cost model, while \textsc{nlocm} and \textsc{lvcp} can also learn the action's dynamics.
Second, we do not require to know the total cost the input plans. %
Unlike them, \optimalapproach{k} can learn the action costs with minimal knowledge, i.e., just a property shared by the input traces, such as  optimality.
However, if available, \optimalapproach{k} can easily incorporate this information as another constraint to the MILP. 
Third, these approaches have been mainly evaluated on syntactic metrics such as precision and recall of the generated model against the ground truth.
This way of evaluating domain learning success has been lately criticized~\cite{behnke2024wipc,garrido2024weep}, as these metrics do not capture the semantic relationship between the original and learned models.
In our case, we prove both theoretically and empirically how \optimalapproach{k} can learn how to generate optimal plans from the input traces.
The resulting models can then be used to solve novel planning tasks in a way that matches the user preferences, i.e., the observed plans.
The same spirit of learning or improving a model in order to generate plans that matches the real-world/user preferences, is present in \cite{lanchas2007}, where they aim to learn the actions' duration from plan execution.
They use relational regression trees to acquire patterns of the states that affect the actions duration.
While they do not provide any theoretical guarantees on the resulting models, we formally prove that the input plans are optimal under the new cost function.
Moreover, \optimalapproach{k} only focuses on the plans rather than in the intermediate states to learn \costfunction.
Finally, while other works can learn user preferences related to the ordering of actions~\cite{seimetz2021learning,sitanskiy2023learning}, we are limited to learning user preferences that can be represented as a cost function.

\paragraph{Inverse Reinforcement Learning.}
Inverse Reinforcement Learning (IRL)~\cite{ng2000algorithms} is the task of inferring the reward function of an agent given its observed behavior.
We can establish a relationship between learning action costs and IRL by assuming that (i) the input plans are observations of that \emph{agent} acting in the environment; and (ii) action costs are the \emph{reward function} we try to learn.
However, IRL differs from learning action costs in few aspects.
First, it is defined over Markov Decision Processes (MDPs), while we learn the action costs in the context of classical planning. 
Second, most IRL approaches assume experts' observations aim at optimizing a single reward function (goal), which is in stark contrast with our setting where every trace may be associated with a different goal. Among the closest works to our setting in the IRL literature we have \cite{choi2012nonparametric} and \cite{michini2012bayesian}. However, the first one assumes trajectories coming from experts belong to clusters each one with a different underlying reward structure, and the second one assumes reward functions can be represented as the combination of simple sub-goals. In our setting, no such assumptions are required.
Finally, to ensure a trace is optimal in an MDP, one must ensure the action taken at each state is optimal, i.e., it has the highest cumulative discounted reward. Therefore, MILP constraints depend solely on the state and actions~\cite{ng2000algorithms}. Conversely, ensuring a plan is optimal in a planning model requires that no cheaper alternative plan exists, so constraints relate to the entire set of alternative plans.

%% file: sections/07-conclusions.tex
\section{Conclusions and Future Work}
We introduced a new problem: that of learning the costs of a set of actions such that a set of input plans are optimal under the resulting planning model.
We have formally proved that this problem does not have a solution for an arbitrary set of input plans, and have relaxed the problem to accept solutions where the number of input plans that are turned optimal is maximized.
We have also presented alternative solutions to these tasks, and 
introduced \optimalapproach{k}, a common algorithm to compute solutions to cost function learning tasks. 
Although the theoretical guarantees are only achieved when $k=\infty$, i.e., all the alternative plans are considered, empirical results show that \optimalapproach{k} can achieve good results in few seconds with lower $k$ values.

In future work, we aim to extend our definitions and algorithms to handle bounded suboptimal plans as inputs. 
Although effective, \optimalapproach{k} scales poorly with increasing planning task size and number. 
We also plan to develop more efficient algorithms that prioritize empirical performance over theoretical guarantees. 
Lastly, we want to evaluate how learned cost functions can be used to detect outliers and concept drift in settings where we receive input plans online.

%% file: appendix.tex
\onecolumn

\section*{Appendix A - \optimalapproach{k} optimality}


\begin{theorem}
    \optimalapproach{k} guarantees \mcf optimal solutions for \cfl tasks when $k = \infty$.
\end{theorem}
\begin{proof}
    In order to prove the optimality of \optimalapproach{k} algorithm, we aim to show that the algorithm performs lexicographic optimization, and that the MILP is properly encoded.
    To prove that \optimalapproach{k} is properly doing lexicographic optimization, we need to demonstrate that it satisfies Definition~5. Observe that lines 3 and 6 of the algorithm correspond to Equations~(9) and (10) of Definition~5 respectively. In line~3, the MILP is solved by setting $\omega_1 = 1$ and $\omega_2=0$, values which are then substituted into Equation~(12). The optimal plans are calculated by the sum of $x_\plan$, which takes the value $1$ for optimal plans and $0$ for not optimal ones. Solving the MILP with these weights will give us a cost function \costfunction that makes $Q$ plans optimal, and therefore, satisfying Definition~5, Equation~(9). Then, the algorithm generates a new MILP with a new constraint (line~5), enforcing that the new solution has to turn exactly $Q$ plans optimal.
    This second MILP is solved, in line~6, with $\omega_1=0$ and $\omega_2=1$ with the goal to find the optimal cost function \costfunction, which makes $Q$ plans optimal and minimizes the sum of action's costs, i.e., corresponding to Equation~(10). 

    We now have to check that the MILP is properly encoded. As mentioned earlier, line~3 is responsible for maximizing the number of optimal plans. It is essential to ensure that the values of the $x_\plan$ variables are properly enforced. This strictly depends on Equation~(14). Observe that since we have $k = \infty$, then we are considering all the alternative plans \alternativeplans{\plan}. The left side of the inequality relies on the value obtained from the sum of the $z_{\plan,\alternativeplan}$ variables, which will be $1$ if plan $\plan$ has a lower or equal cost compared to the alternative plan $\alternativeplan$, and $0$ otherwise. We can get one of the following two cases:
    
    If $\plan$ is an optimal plan, then the sum of the $z$ variables will be equal to $|\alternativeplans{\plan}|$, making the left side of the inequality in Equation~(14) is equal to $0$. 
    To satisfy the inequality, $x_\plan$ can be either $1$ or $0$. However, since the goal is to maximize the number of optimal plans (Equation~(12)), the preferred value for $x_\plan$ is $1$.

    If $\plan$ is not an optimal plan, there exists at least one alternative plan $\alternativeplan$ with a lower or equal cost than $\plan$. Consequently, the left side of the inequality in Equation~(14) results in a positive number greater than $0$. To satisfy the complete inequality, the right side must be greater than or equal to this positive number. This result is achieved by setting $x_\plan$ to $0$ (which aligns with the fact that $\plan$ is not optimal).

    Observe that in order to guarantee that the values of the $x_\plan$ variables are properly enforced, we rely on the correctness of the $z$ variables, as specified in Equation~(13). We consider the following cases:

    If $\plan$ is an optimal plan, then the sum of the actions' costs of $\plan$ is lower than or equal to the sum of the actions' costs of $\alternativeplan$. Regardless of the value that $z_{\plan,\alternativeplan}$ takes, the inequality holds. However, $z_{\plan,\alternativeplan}$ will preferably be $1$ in order to maximize the sum in Equation~(14), and therefore, maximize the number of optimal plans as intended by Equation~(12).

    If $\plan$ is not an optimal plan, for at least one alternative plan $\alternativeplan$, the sum of the actions' costs of $\plan$ is greater than the actions' costs of $\alternativeplan$. To ensure that the inequality in Equation~(13) holds, the correct value for $z_{\plan,\alternativeplan}$ is equal to $0$. This, in turn, influences Equation~(14), ensuring that $x_\plan$ is set correctly and the number of optimal plans in maximized in Equation~(12).
\end{proof}

\begin{remark}
    Note that \optimalapproach{k} can also find \mcf optimal solutions when $k<\infty$.
    However, this is not guaranteed, as the MILP will only ensure each input plan is less costly than a subset of the alternative plans \alternativeplans{\plan}.
\end{remark}

\newpage

\section*{Appendix B - Remaining MILPs}
\subsection*{\textbf{\scf}}

\begin{equation*}
    \small
    \text{maximize} \quad \omega_1 \displaystyle\sum\limits_{\plan \in \multipleplans} x_{\plan} - \omega_2 \displaystyle\sum\limits_{\action \in \relevantactions} y_{\action}
\end{equation*}
\begin{equation*}
    \small
    \text{s.t.} \displaystyle\sum\limits_{\action \in \plan}  y_{\action} + 1 \leq \displaystyle\sum_{\action \in \alternativeplan} y_{\action} + M(1-z_{\plan,\alternativeplan}), \quad  \plan \in \multipleplans, \alternativeplan \in \alternativeplans{\plan}
\end{equation*}
\begin{equation*}
    \small
    |\alternativeplans{\plan}| - \displaystyle\sum_{\alternativeplan \in \alternativeplans{\plan}} z_{\plan,\alternativeplan}   \leq M(1-x_{\plan}), \quad  \plan \in \multipleplans
\end{equation*}
\begin{equation*}
    \small
    x_{\plan}            \in \{0,1\}, \quad \plan \in \multipleplans
\end{equation*}
\begin{equation*}
    \small
    y_{\action}        \geq 1, \quad \action \in \relevantactions
\end{equation*}
\begin{equation*}
    \small
     z_{\plan,\alternativeplan}        \in \{0,1\}, \quad \plan \in \multipleplans, \alternativeplan \in \alternativeplans{\plan}
\end{equation*}

\subsection*{\textbf{\mcfc}}

\begin{equation*}
    \small
    \text{maximize} \quad \omega_1 \displaystyle\sum\limits_{\plan \in \multipleplans} x_{\plan} - \omega_2 \displaystyle\sum\limits_{\action \in \relevantactions} d_{\action}
\end{equation*}
\begin{equation*}
    \small
    \text{s.t.} \displaystyle\sum\limits_{\action \in \plan}  y_{\action} \leq \displaystyle\sum_{\action \in \alternativeplan} y_{\action} + M(1-z_{\plan,\alternativeplan}), \quad  \plan \in \multipleplans, \alternativeplan \in \alternativeplans{\plan}
\end{equation*}
\begin{equation*}
    \small
    |\alternativeplans{\plan}| - \displaystyle\sum_{\alternativeplan \in \alternativeplans{\plan}} z_{\plan,\alternativeplan}   \leq M(1-x_{\plan}), \quad  \plan \in \multipleplans
\end{equation*}
\begin{equation*}
    \small
    \initcostfunction(\action) - y_{\action} \leq d_{\action}, \quad \action \in \relevantactions
\end{equation*}
\begin{equation*}
    \small
    y_{\action} - \initcostfunction(\action) \leq d_{\action}, \quad \action \in \relevantactions
\end{equation*}
\begin{equation*}
    \small
    x_{\plan}            \in \{0,1\}, \quad \plan \in \multipleplans
\end{equation*}
\begin{equation*}
    \small
    y_{\action}        \geq 1, \quad \action \in \relevantactions
\end{equation*}
\begin{equation*}
    \small
    d_{\action}        \geq 0, \quad \action \in \relevantactions
\end{equation*}
\begin{equation*}
    \small
     z_{\plan,\alternativeplan}        \in \{0,1\}, \quad \plan \in \multipleplans, \alternativeplan \in \alternativeplans{\plan}
\end{equation*}

\subsection*{\textbf{\scfc}}

\begin{equation*}
    \small
    \text{maximize} \quad \omega_1 \displaystyle\sum\limits_{\plan \in \multipleplans} x_{\plan} - \omega_2 \displaystyle\sum\limits_{\action \in \relevantactions} d_{\action}
\end{equation*}
\begin{equation*}
    \small
    \text{s.t.} \displaystyle\sum\limits_{\action \in \plan}  y_{\action} + 1 \leq \displaystyle\sum_{\action \in \alternativeplan} y_{\action} + M(1-z_{\plan,\alternativeplan}), \quad  \plan \in \multipleplans, \alternativeplan \in \alternativeplans{\plan}
\end{equation*}
\begin{equation*}
    \small
    |\alternativeplans{\plan}| - \displaystyle\sum_{\alternativeplan \in \alternativeplans{\plan}} z_{\plan,\alternativeplan}   \leq M(1-x_{\plan}), \quad  \plan \in \multipleplans
\end{equation*}
\begin{equation*}
    \small
    \initcostfunction(\action) - y_{\action} \leq d_{\action}, \quad \action \in \relevantactions
\end{equation*}
\begin{equation*}
    \small
    y_{\action} - \initcostfunction(\action) \leq d_{\action}, \quad \action \in \relevantactions
\end{equation*}
\begin{equation*}
    \small
    x_{\plan}            \in \{0,1\}, \quad \plan \in \multipleplans
\end{equation*}
\begin{equation*}
    \small
    y_{\action}        \geq 1, \quad \action \in \relevantactions
\end{equation*}
\begin{equation*}
    \small
    d_{\action}        \geq 0, \quad \action \in \relevantactions
\end{equation*}
\begin{equation*}
    \small
     z_{\plan,\alternativeplan}        \in \{0,1\}, \quad \plan \in \multipleplans, \alternativeplan \in \alternativeplans{\plan}
\end{equation*}

\newpage

\section*{Appendix C - Experiments}
\input{tikz_figures/experiments-scf}
\input{tikz_figures/experiments-mcfc}
\input{tikz_figures/experiments-scfc}
\newpage
\input{figures/execution_time_scf}
\input{figures/execution_time_mcfc}
\input{figures/execution_time_scfc}

%% file: tikz_figures/experiments-scf.tex
\begin{table*}[h]\label{tab:ratio-optimal-scf}
    \centering
    \renewcommand{\tabcolsep}{0.06cm}
    \scalebox{0.7}{\begin{tabular}{|c|c|c|c|c|c|c|c|c|c|c|c|c|}
     \hline
      & \multicolumn{3}{c|}{\textsc{grid}} & \multicolumn{3}{c|}{\textsc{barman}} & \multicolumn{3}{c|}{\textsc{openstacks}} & \multicolumn{3}{c|}{\textsc{transport}} \\ \hline
      & $|\cfl|=10^1$ & $|\cfl|=10^2$ & $|\cfl|=10^3$ & $|\cfl|=10$ & $|\cfl|=10^2$ & $|\cfl|=10^3$ & $|\cfl|=10^1$ & $|\cfl|=10^2$ & $|\cfl|=10^3$ & $|\cfl|=10^1$ & $|\cfl|=10^2$ & $|\cfl|=10^3$ \\ \hline
      \textsc{baseline} &\cellcolor{YellowGreen!7}{$0.02 \pm 0.04$}&\cellcolor{YellowGreen!7}{$0.04 \pm 0.05$}&\cellcolor{YellowGreen!25}{$0.15 \pm 0.13$}&\cellcolor{YellowGreen!7}{$0.00 \pm 0.00$}&\cellcolor{gray!70}{}&\cellcolor{YellowGreen!7}{$0.00 \pm 0.00$}&\cellcolor{YellowGreen!7}{$0.00 \pm 0.00$}&\cellcolor{YellowGreen!7}{$0.00 \pm 0.00$}&\cellcolor{YellowGreen!7}{$0.00 \pm 0.00$}&\cellcolor{YellowGreen!7}{$0.00 \pm 0.00$}&\cellcolor{YellowGreen!7}{$0.00 \pm 0.00$}&\cellcolor{YellowGreen!7}{$0.00 \pm 0.00$}\\ \hline
      
      $k = 10^1$ &\cellcolor{YellowGreen!15}{$0.08 \pm 0.12$}&\cellcolor{YellowGreen!7}{$0.03 \pm 0.05$}&\cellcolor{YellowGreen!15}{$0.14 \pm 0.10$}&\cellcolor{YellowGreen!7}{$0.00 \pm 0.00$}&\cellcolor{gray!70}{}&\cellcolor{YellowGreen!7}{$0.00 \pm 0.00$}&\cellcolor{YellowGreen!7}{$0.00 \pm 0.00$}&\cellcolor{YellowGreen!7}{$0.00 \pm 0.00$}&\cellcolor{YellowGreen!7}{$0.00 \pm 0.00$}&\cellcolor{YellowGreen!7}{$0.00 \pm 0.00$}&\cellcolor{YellowGreen!15}{$0.10 \pm 0.30$}& \cellcolor{YellowGreen!7}{$0.00 \pm 0.00$}\\ \hline
      
      $k = 10^2$&\cellcolor{YellowGreen!55}{$0.50 \pm 0.25$}&\cellcolor{gray!70}{}&\cellcolor{YellowGreen!15}{$0.14 \pm 0.10$}&\cellcolor{YellowGreen!7}{$0.00 \pm 0.00$}&\cellcolor{gray!70}{}&\cellcolor{YellowGreen!7}{$0.00 \pm 0.00$}&\cellcolor{YellowGreen!7}{$0.00 \pm 0.00$}&\cellcolor{YellowGreen!7}{$0.00 \pm 0.00$}&\cellcolor{YellowGreen!7}{$0.00 \pm 0.00$}&\cellcolor{YellowGreen!7}{$0.00 \pm 0.00$}&\cellcolor{gray!70}&\cellcolor{gray!70} \\ \hline
      
      $k = 10^3$ &\cellcolor{YellowGreen!85}{$0.80 \pm 0.11$}&\cellcolor{gray!70}{}&\cellcolor{gray!70}{}&\cellcolor{YellowGreen!7}{$0.00 \pm 0.00$}&\cellcolor{gray!70}{}&\cellcolor{YellowGreen!7}{$0.00 \pm 0.00$}&\cellcolor{gray!70}{}&\cellcolor{gray!70}{}&\cellcolor{gray!70}{}&\cellcolor{gray!70}{}&\cellcolor{gray!70}&\cellcolor{YellowGreen!7}{$0.00 \pm 0.00$}\\ \hline
      
      $k = 10^4$&\cellcolor{YellowGreen!85}{$0.76 \pm 0.20$}&\cellcolor{YellowGreen!35}{$0.32 \pm 0.07$}&\cellcolor{gray!70}{}&\cellcolor{gray!70}{}&\cellcolor{gray!70}{}&\cellcolor{gray!70}{}&\cellcolor{gray!70}{}&\cellcolor{gray!70}{}&\cellcolor{gray!70}{}&\cellcolor{gray!70}{}&\cellcolor{gray!70}& \cellcolor{gray!70}\\ \hline
    \end{tabular}}
    \caption{ Mean and standard deviation of the ratio of plans turned optimal when the different algorithms compute \scf solutions.}
\end{table*}

%% file: tikz_figures/experiments-mcfc.tex
\begin{table*}[h]\label{tab:ratio-optimal-mcfc}
    \centering
    \renewcommand{\tabcolsep}{0.06cm}
    \scalebox{0.7}{\begin{tabular}{|c|c|c|c|c|c|c|c|c|c|c|c|c|}
     \hline
      & \multicolumn{3}{c|}{\textsc{grid}} & \multicolumn{3}{c|}{\textsc{barman}} & \multicolumn{3}{c|}{\textsc{openstacks}} & \multicolumn{3}{c|}{\textsc{transport}} \\ \hline
      & $|\cfl|=10^1$ & $|\cfl|=10^2$ & $|\cfl|=10^3$ & $|\cfl|=10$ & $|\cfl|=10^2$ & $|\cfl|=10^3$ & $|\cfl|=10^1$ & $|\cfl|=10^2$ & $|\cfl|=10^3$ & $|\cfl|=10^1$ & $|\cfl|=10^2$ & $|\cfl|=10^3$ \\ \hline
       \textsc{baseline} &\cellcolor{YellowGreen!45}{$0.40 \pm 0.22$} &\cellcolor{YellowGreen!55}{$0.46 \pm 0.00$} &\cellcolor{YellowGreen!45}{$0.37 \pm 0.05$} &\cellcolor{YellowGreen!25}{$0.20 \pm 0.00$}&\cellcolor{YellowGreen!25}{$0.18 \pm 0.30$}&\cellcolor{YellowGreen!15}{$0.06 \pm 0.04$}&\cellcolor{YellowGreen!7}{$0.03 \pm 0.05$}&\cellcolor{YellowGreen!7}{$0.01 \pm 0.01$}&\cellcolor{YellowGreen!7}{$0.01 \pm 0.00$}&\cellcolor{YellowGreen!15}{$0.07 \pm 0.05$}&\cellcolor{YellowGreen!7}{$0.02 \pm 0.00$}&\cellcolor{YellowGreen!7}{$0.03 \pm 0.00$}\\ \hline
      
      $k = 10^1$ &\cellcolor{YellowGreen!45}{$0.43 \pm 0.24$}&\cellcolor{YellowGreen!35}{$0.26 \pm 0.00$} &\cellcolor{YellowGreen!45}{$0.37 \pm 0.05$} &\cellcolor{YellowGreen}{$1.00 \pm 0.00$}&\cellcolor{YellowGreen!65}{$0.65 \pm 0.27$}&\cellcolor{YellowGreen!15}{$0.08 \pm 0.03$}&\cellcolor{YellowGreen!7}{$0.00 \pm 0.00$}&\cellcolor{YellowGreen!7}{$0.00 \pm 0.00$}&\cellcolor{YellowGreen!7}{$0.00 \pm 0.00$}&\cellcolor{YellowGreen!15}{$0.12 \pm 0.04$}&\cellcolor{YellowGreen!15}{$0.07 \pm 0.00$}&\cellcolor{YellowGreen!7}{$0.00 \pm 0.00$}\\ \hline
      
      $k = 10^2$&\cellcolor{YellowGreen!75}{$0.73 \pm 0.11$} & \cellcolor{YellowGreen!45}{$0.35 \pm 0.00$}&\cellcolor{YellowGreen!45}{$0.37 \pm 0.05$} &\cellcolor{YellowGreen}{$1.00 \pm 0.00$}&\cellcolor{YellowGreen!65}{$0.65 \pm 0.27$}&\cellcolor{YellowGreen!15}{$0.08 \pm 0.03$}&\cellcolor{YellowGreen!7}{$0.00 \pm 0.00$}&\cellcolor{YellowGreen!7}{$0.00 \pm 0.00$}&\cellcolor{YellowGreen!7}{$0.00 \pm 0.00$}&\cellcolor{YellowGreen!35}{$0.28 \pm 0.15$}&\cellcolor{YellowGreen!15}{$0.13 \pm 0.00$}&\cellcolor{YellowGreen!15}{$0.07 \pm 0.00$} \\ \hline
      
      $k = 10^3$ &\cellcolor{YellowGreen!85}{$0.83 \pm 0.07$} &\cellcolor{YellowGreen!55}{$0.49 \pm 0.00$} &\cellcolor{YellowGreen!45}{$0.37 \pm 0.05$} &\cellcolor{YellowGreen}{$1.00 \pm 0.00$}&\cellcolor{YellowGreen!65}{$0.65 \pm 0.27$}&\cellcolor{YellowGreen!15}{$0.08 \pm 0.03$}&\cellcolor{gray!70}{}&\cellcolor{gray!70}{}&\cellcolor{gray!70}{}&\cellcolor{gray!70}{}&\cellcolor{gray!70}{}&\cellcolor{gray!70}{} \\ \hline
      
      $k = 10^4$&\cellcolor{YellowGreen!65}{$0.60 \pm 0.40$}&\cellcolor{YellowGreen!55}{$0.49 \pm 0.00$} &\cellcolor{YellowGreen!45}{$0.37 \pm 0.05$} &\cellcolor{YellowGreen}{$1.00 \pm 0.00$}&\cellcolor{gray!70}{}&\cellcolor{gray!70}{}&\cellcolor{gray!70}{}&\cellcolor{gray!70}{}&\cellcolor{gray!70}{}&\cellcolor{gray!70}{}&\cellcolor{gray!70}{}&\cellcolor{gray!70}{}\\ \hline
    \end{tabular}}
    \caption{ Mean and standard deviation of the ratio of plans turned optimal when the different algorithms compute \mcfc solutions.}
\end{table*}

%% file: tikz_figures/experiments-scfc.tex
\begin{table*}[h]\label{tab:ratio-optimal-mcfc}
    \centering
    \renewcommand{\tabcolsep}{0.06cm}
    \scalebox{0.7}{\begin{tabular}{|c|c|c|c|c|c|c|c|c|c|c|c|c|}
     \hline
      & \multicolumn{3}{c|}{\textsc{grid}} & \multicolumn{3}{c|}{\textsc{barman}} & \multicolumn{3}{c|}{\textsc{openstacks}} & \multicolumn{3}{c|}{\textsc{transport}} \\ \hline
      & $|\cfl|=10^1$ & $|\cfl|=10^2$ & $|\cfl|=10^3$ & $|\cfl|=10$ & $|\cfl|=10^2$ & $|\cfl|=10^3$ & $|\cfl|=10^1$ & $|\cfl|=10^2$ & $|\cfl|=10^3$ & $|\cfl|=10^1$ & $|\cfl|=10^2$ & $|\cfl|=10^3$ \\ \hline
      \textsc{baseline} &\cellcolor{YellowGreen!7}{$0.03 \pm 0.04$}&\cellcolor{YellowGreen!7}{$0.00 \pm 0.00$}&\cellcolor{YellowGreen!25}{$0.21 \pm 0.15$}&\cellcolor{YellowGreen!7}{$0.00 \pm 0.00$}&\cellcolor{gray!70}{}&\cellcolor{YellowGreen!7}{$0.00 \pm 0.00$}&\cellcolor{YellowGreen!7}{$0.00 \pm 0.00$}&\cellcolor{YellowGreen!7}{$0.00 \pm 0.00$}&\cellcolor{YellowGreen!7}{$0.00 \pm 0.00$}&\cellcolor{YellowGreen!7}{$0.00 \pm 0.00$}&\cellcolor{YellowGreen!7}{$0.00 \pm 0.00$}&\cellcolor{YellowGreen!7}{$0.00 \pm 0.00$}\\ \hline
      
      $k = 10^1$ &\cellcolor{YellowGreen!15}{$0.07 \pm 0.13$}&\cellcolor{YellowGreen!7}{$0.00 \pm 0.00$}&\cellcolor{YellowGreen!65}{$0.57 \pm 0.43$}&\cellcolor{YellowGreen!7}{$0.00 \pm 0.00$}&\cellcolor{gray!70}{}&\cellcolor{YellowGreen!7}{$0.00 \pm 0.00$}&\cellcolor{YellowGreen!7}{$0.00 \pm 0.00$}&\cellcolor{YellowGreen!25}{$0.20 \pm 0.40$}&\cellcolor{YellowGreen!7}{$0.00 \pm 0.00$}&\cellcolor{YellowGreen!7}{$0.02 \pm 0.04$}&\cellcolor{YellowGreen!15}{$0.10 \pm 0.30$}&\cellcolor{YellowGreen!45}{$0.42 \pm 0.47$}\\ \hline
      
      $k = 10^2$&\cellcolor{YellowGreen!65}{$0.55 \pm 0.18$}&\cellcolor{gray!70}{}&\cellcolor{YellowGreen!25}{$0.19 \pm 0.11$}&\cellcolor{YellowGreen!7}{$0.00 \pm 0.00$}&\cellcolor{gray!70}{}&\cellcolor{YellowGreen!7}{$0.00 \pm 0.00$}&\cellcolor{YellowGreen!15}{$0.10 \pm 0.30$}&\cellcolor{YellowGreen!25}{$0.20 \pm 0.40$}&\cellcolor{YellowGreen!15}{$0.14 \pm 0.35$}&\cellcolor{YellowGreen!25}{$0.22 \pm 0.07$}&\cellcolor{gray!70}{}&\cellcolor{gray!70}{} \\ \hline
      
      $k = 10^3$ &\cellcolor{YellowGreen!85}{$0.80 \pm 0.07$}&\cellcolor{YellowGreen!25}{$0.19 \pm 0.00$}&\cellcolor{gray!70}{}&\cellcolor{YellowGreen!7}{$0.00 \pm 0.00$}&\cellcolor{gray!70}{}&\cellcolor{YellowGreen!7}{$0.00 \pm 0.00$}&\cellcolor{gray!70}{}&\cellcolor{gray!70}{}&\cellcolor{gray!70}{}&\cellcolor{gray!70}{}&\cellcolor{gray!70}{}&\cellcolor{gray!70}{} \\ \hline
      
      $k = 10^4$&\cellcolor{YellowGreen!85}{$0.80 \pm 0.00$}&\cellcolor{YellowGreen!15}{$0.08 \pm 0.00$}&\cellcolor{gray!70}{}&\cellcolor{gray!70}{}&\cellcolor{gray!70}{}&\cellcolor{gray!70}{}&\cellcolor{gray!70}{}&\cellcolor{gray!70}{}&\cellcolor{gray!70}{}&\cellcolor{gray!70}{}&\cellcolor{gray!70}{}&\cellcolor{gray!70}{}\\ \hline
    \end{tabular}}
    \caption{ Mean and standard deviation of the ratio of plans turned optimal when the different algorithms compute \scfc solutions.}
\end{table*}

%% file: figures/execution_time_scf.tex
\begin{figure*}
    \centering
    \begin{subfigure}[t]{0.33\textwidth}\centering
    \includegraphics[width=\textwidth]{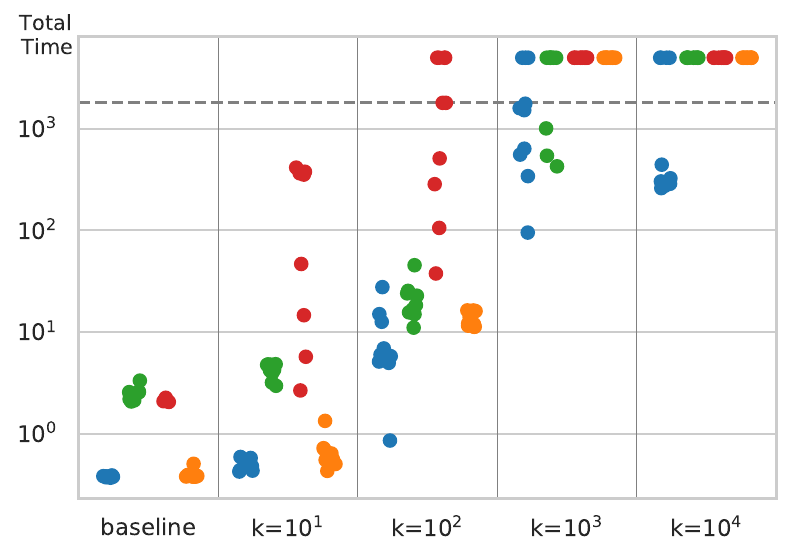}
    \caption{$|\cfl|=10^1$}
    \label{fig:cfl10}
    \end{subfigure}
    \hfill
    \begin{subfigure}[t]{0.33\textwidth}\centering
    \includegraphics[width=\textwidth]{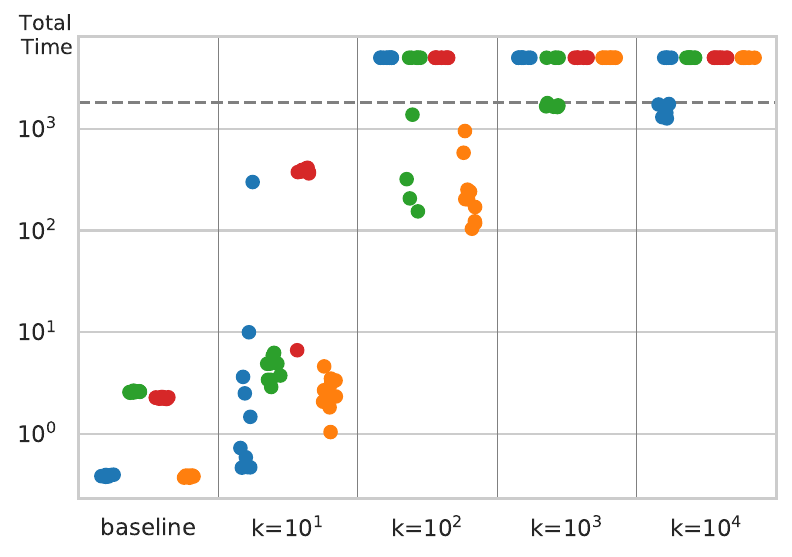}
    \caption{$|\cfl|=10^2$}
    \label{fig:cfl100}
    \end{subfigure}
    \hfill
    \begin{subfigure}[t]{0.33\textwidth}\centering
    \includegraphics[width=\textwidth]{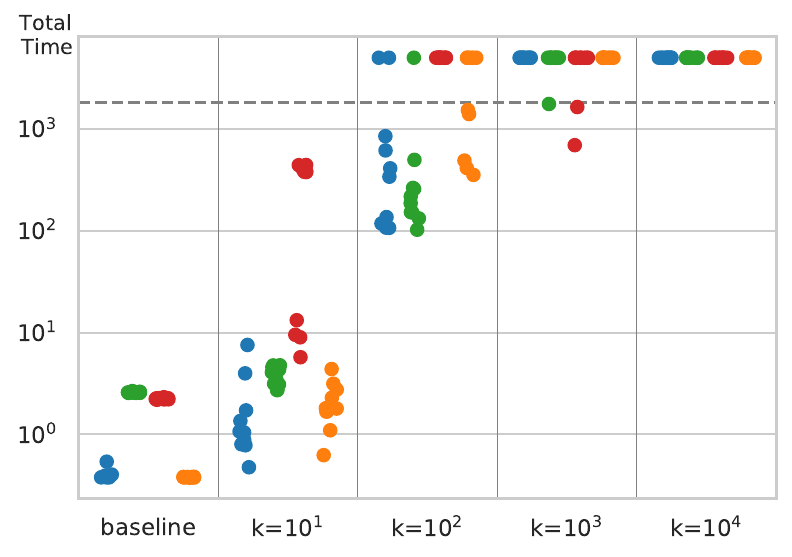}
    \caption{$|\cfl|=10^3$}
    \label{fig:cfl1000}
    \end{subfigure}
    \vspace{3mm}
    \caption{Total execution time of the algorithms computing \scf solutions for the four domains across different numbers of alternative plans $k$, with each plot corresponding to one of the CFL task sizes. 
    Domains are represented by different colors: \textcolor{blue}{\textsc{Grid}}, \textcolor{Green}{\textsc{Barman}}, \textcolor{orange}{\textsc{Openstacks}}, and \textcolor{red}{\textsc{Transport}}.}
\end{figure*}

%% file: figures/execution_time_mcfc.tex
\begin{figure*}
    \centering
    \begin{subfigure}[t]{0.33\textwidth}\centering
    \includegraphics[width=\textwidth]{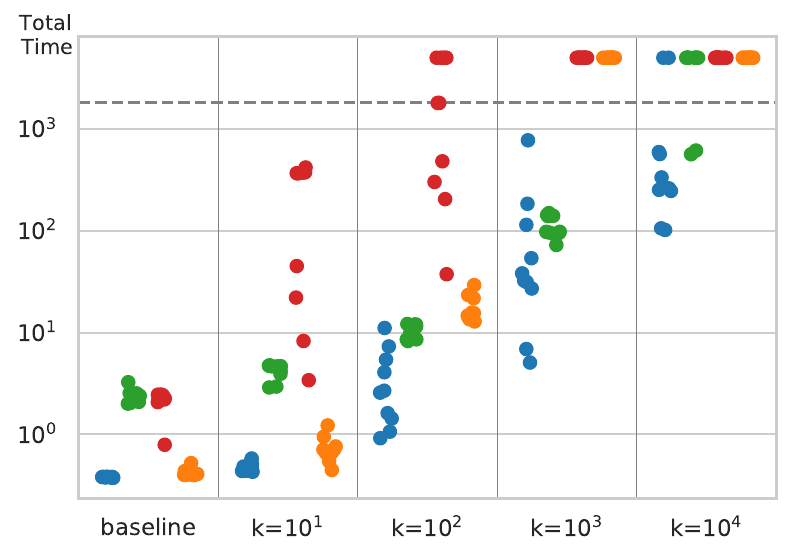}
    \caption{$|\cfl|=10^1$}
    \label{fig:cfl10}
    \end{subfigure}
    \hfill
    \begin{subfigure}[t]{0.33\textwidth}\centering
    \includegraphics[width=\textwidth]{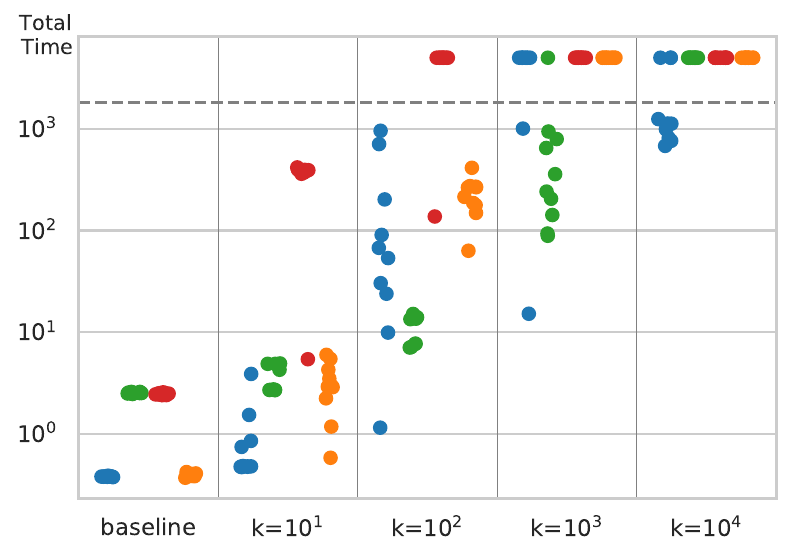}
    \caption{$|\cfl|=10^2$}
    \label{fig:cfl100}
    \end{subfigure}
    \hfill
    \begin{subfigure}[t]{0.33\textwidth}\centering
    \includegraphics[width=\textwidth]{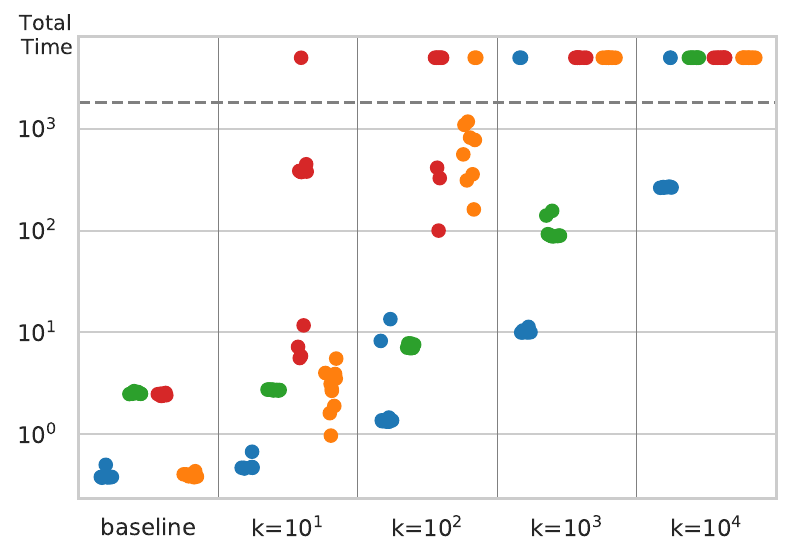}
    \caption{$|\cfl|=10^3$}
    \label{fig:cfl1000}
    \end{subfigure}
    \vspace{3mm}
    \caption{Total execution time of the algorithms computing \mcfc solutions for the four domains across different numbers of alternative plans $k$, with each plot corresponding to one of the CFL task sizes. 
    Domains are represented by different colors: \textcolor{blue}{\textsc{Grid}}, \textcolor{Green}{\textsc{Barman}}, \textcolor{orange}{\textsc{Openstacks}}, and \textcolor{red}{\textsc{Transport}}.}
\end{figure*}

%% file: figures/execution_time_scfc.tex
\begin{figure*}
    \centering
    \begin{subfigure}[t]{0.33\textwidth}\centering
    \includegraphics[width=\textwidth]{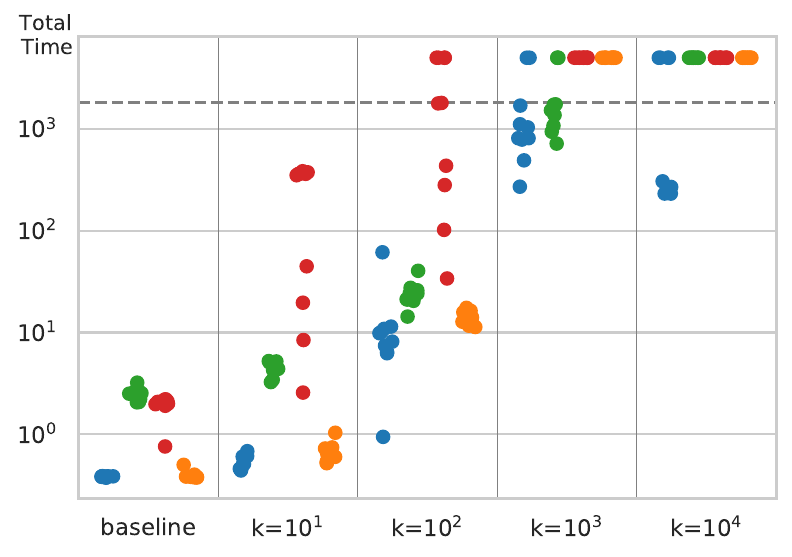}
    \caption{$|\cfl|=10^1$}
    \label{fig:cfl10}
    \end{subfigure}
    \hfill
    \begin{subfigure}[t]{0.33\textwidth}\centering
    \includegraphics[width=\textwidth]{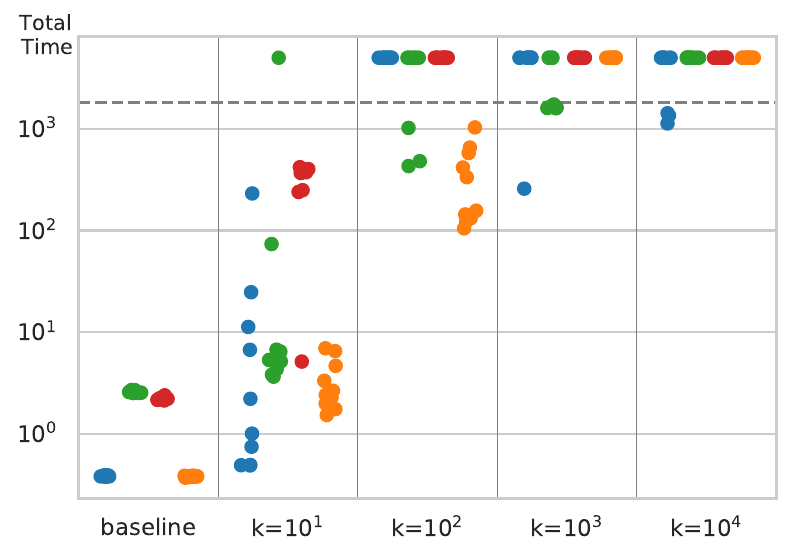}
    \caption{$|\cfl|=10^2$}
    \label{fig:cfl100}
    \end{subfigure}
    \hfill
    \begin{subfigure}[t]{0.33\textwidth}\centering
    \includegraphics[width=\textwidth]{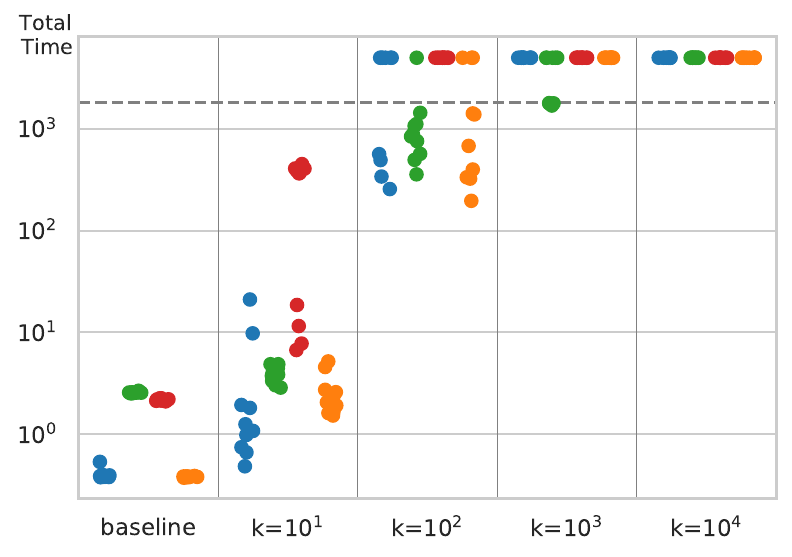}
    \caption{$|\cfl|=10^3$}
    \label{fig:cfl1000}
    \end{subfigure}
    \vspace{3mm}
    \caption{Total execution time of the algorithms computing \scfc solutions for the four domains across different numbers of alternative plans $k$, with each plot corresponding to one of the CFL task sizes. 
    Domains are represented by different colors: \textcolor{blue}{\textsc{Grid}}, \textcolor{Green}{\textsc{Barman}}, \textcolor{orange}{\textsc{Openstacks}}, and \textcolor{red}{\textsc{Transport}}.}
\end{figure*}